\documentclass{article}
\usepackage[margin=1.2in]{geometry}
\usepackage{latexsym, amsthm, amscd, amsfonts, mathrsfs, amsmath, amssymb, stmaryrd, tikz-cd, mathrsfs, bbm, url, esint, mathtools, enumerate, caption, subcaption}
\usepackage{dsfont}
\usepackage{blkarray}
\usepackage[utf8]{inputenc} 
\usepackage[T1]{fontenc}    
\usepackage{hyperref}       
\usepackage{url}            
\usepackage{booktabs}       
\usepackage{nicefrac}       
\usepackage{microtype}      
\usepackage{xcolor}  
\usepackage{algorithm}
\usepackage{algpseudocode}
\usepackage{enumitem}
\usepackage{comment}

\usepackage{natbib}

\newcommand{\Cov}{\mathrm{cov}}

\newcommand{\pr}{\mathbb{P}}
\newcommand{\E}{\mathbb{E}}
\DeclareMathOperator{\CP}{CP}
\DeclareMathOperator{\Var}{Var}

\DeclareMathOperator{\argmin}{argmin}

\DeclareMathOperator{\ReLU}{ReLU}
\DeclareMathOperator{\Unif}{Unif}

\DeclareMathOperator{\sign}{sgn}

\DeclareMathOperator{\orb}{orb}

\DeclareMathOperator{\NN}{NN}

\DeclareMathOperator{\TV}{TV}
\DeclareMathOperator{\Inf}{Inf}
\newcommand{\cF}{\mathcal{F}}
\newcommand{\cX}{\mathcal{X}}
\newcommand{\cV}{\mathcal{V}}
\newcommand{\cU}{\mathcal{U}}

\newcommand{\cN}{\mathcal{N}}

\newcommand{\bR}{\mathbb{R}}

\newcommand{\bN}{\mathbb{N}}

\newcommand{\cL}{\mathcal{L}}

\newcommand{\cS}{\mathcal{S}}

\definecolor{mydarkblue}{rgb}{0,0.08,0.45}
  \hypersetup{ %
    colorlinks=true,
    linkcolor=mydarkblue,
    citecolor=mydarkblue,
    filecolor=mydarkblue,
    urlcolor=mydarkblue,
    }
\definecolor{DSgray}{cmyk}{0,0,0,0.7}
\definecolor{DSred}{cmyk}{0,0.7,0,0.7}

\theoremstyle{plain}
\newtheorem{definition}{Definition}
\theoremstyle{plain}
\theoremstyle{plain}
\newtheorem{theorem}{Theorem}
\theoremstyle{plain}
\theoremstyle{plain}
\theoremstyle{plain}
\theoremstyle{plain}
\newtheorem{proposition}{Proposition}
\theoremstyle{plain}
\theoremstyle{plain}
\newtheorem{lemma}{Lemma}
\theoremstyle{plain}
\newtheorem{corollary}{Corollary}
\theoremstyle{plain}
\theoremstyle{remark}
\newtheorem{remark}{Remark}
\theoremstyle{remark}
\theoremstyle{plain}
\usepackage[affil-it]{authblk} 

\title{The Benefits of Temporal Correlations:\\
SGD Learns $k$-Juntas from Random Walks Efficiently}

\author[1]{
Elisabetta Cornacchia}
\author[2]{Dan Mikulincer}
\author[3]{Elchanan Mossel}

\affil[1]{\small Bocconi University}
\affil[2]{\small University of Washington (UW).}
\affil[3]{\small Massachusetts Institute of Technology (MIT).}

\usepackage[parfill]{parskip}

\date{}

\begin{document}

\maketitle

\begingroup
    \renewcommand\thefootnote{}  
\footnotetext{\small 
\noindent
Authors are listed in alphabetical order. Email: \texttt{elisabetta.cornacchia@unibocconi.it}, \texttt{danmiku@uw.edu}, \texttt{elmos@mit.edu}.}
\endgroup

\begin{abstract}
  We study how temporal correlations in the data can make certain sparse learning problems efficiently learnable by gradient-based methods. Our focus is on Boolean \(k\)-juntas, a canonical sparse learning problem known to pose barriers for gradient-based methods under independent uniform samples. We show that this picture changes when the samples are generated by a lazy random walk on the hypercube. In this setting, the temporal dependencies can be exploited by a two-layer ReLU network trained using stylized-SGD with a temporal-difference loss, which compares target and predicted increments across consecutive samples. For every fixed $k$, the resulting sample complexity is essentially linear in the ambient dimension $d$. By contrast, we show that for large-batch gradient methods using standard convex pointwise losses, temporal correlations do not provide the same advantage. 
\end{abstract}

\section{Introduction}

The empirical success of neural networks continues to outpace our theoretical understanding. One way to narrow this gap is to study simplified settings in which gradient-based methods can be shown to learn efficiently~\cite{du2019gradient,damian2022neural,abbe2022merged,arous2021online}. Much of the existing work in this direction falls into one of two extremes: worst-case models, such as distribution-free PAC learning, where no statistical structure is assumed~\cite{valiant1984theory,shalev2014understanding}, and idealized stochastic models, where examples are drawn independently from a fixed distribution~\cite{arous2021online,bietti2020deep,abbe2023sgd,allen2019convergence,dandi2024benefits,bietti2022learning,glasgow2023sgd}. While both viewpoints have yielded important insights, they leave open an intermediate regime that is arguably closer to many real data sources: samples are neither adversarially chosen nor independent, but exhibit structured correlations.

Temporal correlations are a particularly common form of such structure. They arise in domains such as computer vision~\cite{ravanbakhsh2024deep}, healthcare~\cite{yao2022wild}, and environmental modeling~\cite{amato2020novel}, where data are generated sequentially and often display substantial temporal dependence or distribution shift. In this work, we study an idealized and tractable model of temporal dependence, focusing on the random walk on the hypercube. Our goal is to understand whether such temporal structure can itself influence the performance of gradient-based learning, complementing prior work that has focused on the role of spatial or geometric structure~\cite{goldt2020modeling,gerace2020generalisation}.

To study the effect of temporal correlations, we focus on learning Boolean \(k\)-juntas.
A junta is a function \(f:\{\pm1\}^d\to\mathbb{R}\) that depends only on a small subset of the coordinates: that is, there exists a set \(S\subseteq[d]\) with \(|S|=k\), and \(k=O(1)\), such that
\(
f(x)=\tilde f(x_S),
\)
where \(x_S\) denotes the restriction of \(x\) to the coordinates in \(S\). This problem has attracted significant recent attention as a canonical model of sparse feature learning, and has served as a useful testbed to rigorously study phenomena such as incremental feature learning~\cite{abbe2022merged,abbe2023generalization}, separations between loss functions~\cite{joshi2024complexity}, curriculum learning~\cite{cornacchia2023mathematical,abbe2023provable}, and knowledge distillation~\cite{panigrahi2024progressive}. In the \emph{i.i.d.} uniform setting, previous works have established complexity \emph{lower bounds} that scale unfavorably with $k$, and showed that several natural algorithms, including gradient-based training of standard neural networks, require $d^{\Omega(k)}$ time
~\cite{daniely2020learning,AS20,abbe2023sgd,shalev2017failures}.

In contrast, we study learning \(f\) from samples generated by a random walk on the hypercube. The resulting temporal correlations allow gradient-based methods to bypass barriers that are specific to independent uniform samples. Our main result shows that, with an appropriate temporal-difference loss, a two-layer ReLU network trained by a stochastic gradient-based method can learn any \(k\)-junta using a number of samples essentially linear in the ambient dimension \(d\), for every fixed \(k\). 
Interestingly, our theory and experiments suggest that whether this advantage is realized depends crucially on both the choice of loss function and the batch size.

\section{Summary of Contributions}
To explain our results, it is helpful to begin with the special case of a
\(k\)-parity,
\(
f(x)=\prod_{i\in S}x_i ,
\)
where \(S\subseteq[d]\), $|S|=k$, is the unknown support. 
Although parities are
efficiently learnable by specialized algorithms, such as Gaussian elimination
over \(\mathbb F_2\), they pose a well-known barrier for gradient-based methods
under \emph{i.i.d.} uniform inputs. In particular, for gradient descent with
limited gradient precision, \(d^{\Omega(k)}\) steps are needed in this setting
~\cite{daniely2020learning,AS20,barak2022hidden,kou2024matching}. The intuition is that, under the uniform distribution,
\(\E[f(x)x_j]=0 \) for every \(j\in[d],
\)
so first-order pointwise correlations do not reveal the support of the target
parity. 

The situation changes when the samples are generated by a random walk on the
hypercube. Since each step changes at most one coordinate, the finite difference
\[
\Delta f^{(t)}:=f(x_{t+1})-f(x_t)
\]
is nonzero only when the updated coordinate belongs to \(S\). Thus, the support
can be recovered by a simple procedure that records which coordinate flips
produce nonzero label increments and a coupon-collector argument shows that this
takes \(\Theta(d\log k)\) steps. Our goal, however, is not to rely on this
explicit problem-specific procedure. Rather, we ask whether gradient-based
methods can extract the same information automatically.



Our main result shows that when a neural network is trained by a stochastic gradient-based algorithm with a suitable loss, the updates themselves pick up the information carried by the temporal differences, thus one does not need to recover the support by hand. To go beyond parities, we must consider the influences $\mathrm{Inf}_i(f)$ within the support when $i \in S$, to ensure the finite differences carry enough information, see Section \ref{sec:results} for the exact definition. We state the following informal result for $k$-juntas.
\begin{theorem}[Informal] \label{thm:informal}
Let $f:\{\pm1\}^d \to \mathbb{R}$ be a $k$-junta, supported on $S \subset [d]$, let $(x_t)_{t=1}^T$ be a lazy random walk on $\{\pm1\}^d$, and let $\varepsilon > 0$. Suppose that $\min_{i\in S}\mathrm{Inf}_i(f)\geq \tau$, for some $\tau > 0$, and that $T = \tilde\Omega_{k,\varepsilon,\tau}(d)$. Then, a two-layer neural network, trained using layerwise SGD with suitable loss (see Algorithm \ref{alg:layerwise-sgd}), on the samples $(x_t, f(x_t))_{t=1}^T$, achieves
\[
\E_{x \sim \mathrm{Unif}(\{\pm1\}^d)} \bigl[(\NN(x) - f(x))^2\bigr] < \varepsilon.
\]
\end{theorem}

The mechanism behind the result is a second-moment effect created by the
random-walk dynamics. For a parity of degree at least two, the signed quantity
\[
\Delta f^{(t)} \Delta x_j^{(t)},
\]
where $\Delta x_j^{(t)}:=x_j^{(t)}-x_j^{(t-1)} $, still has mean zero, so there is no informative first-order signal. However, its
second moment distinguishes relevant and irrelevant coordinates: it is positive
for \(j\in S\) and zero for \(j\notin S\). Thus, temporal differences induce a
support-dependent anisotropy that is absent under independent uniform sampling.
The temporal-difference loss~(see Def.~\ref{def:TDloss}) is designed to expose this anisotropy to the
gradient updates.

In Section~\ref{sec:results}, we make Theorem~\ref{thm:informal} precise and
give the full layerwise SGD procedure. Up to problem-dependent constants and lgoarithmic terms, the
sample complexity of the resulting gradient-based algorithm matches the
\(\Theta(d\log k)\) scale of the explicit support-recovery procedure. Compared
with the \(d^{\Omega(k)}\) barriers known for independent uniform samples, this
shows that temporal correlations can dramatically change the tractability of
sparse learning problems.

Our second contribution is a complementary lower bound showing that temporal
correlations are not automatically useful for all gradient-based methods. In
Section~\ref{sec:lower_bound}, we prove that noisy gradient descent with large
batch size and standard convex pointwise losses, such as the square or hinge
loss, remains controlled by a cross-predictability barrier analogous to the
\emph{i.i.d.} setting. Thus, in the large-batch regime, pointwise losses do not
substantially exploit the temporal structure. By contrast, this lower bound does
not apply to small-batch SGD, and our experiments suggest that small-batch
square-loss SGD may sometimes exploit temporal dependencies as well. We discuss
this phenomenon in Section~\ref{sec:numerical}.

\section{Related Works}
\paragraph{Learning juntas from i.i.d. data.}
Several works considered the problem of learning juntas, and in particular sparse parities, under \emph{i.i.d. uniform} inputs, as a testbed for understanding the computational limits of gradient descent methods. 
For parity functions, stochastic gradient descent (SGD) with small batch size can learn the target efficiently on an emulation network~\citep{abbe2020poly}. By contrast, in the Statistical Query (SQ) model~\citep{kearns1998efficient}, and similarly for gradient descent with limited gradient precision, any such learning procedure requires at least \(\Omega\!\left(d^k\right)\) time~\citep{abbe2020poly}. For standard shallow networks, positive results have been established for sparse parities, matching the SQ lower bound~\citep{barak2022hidden,glasgow2023sgd,kou2024matching}. For juntas, the complexity depends on the Fourier-Walsh structure of the target function, as captured by the \emph{leap}~\cite{abbe2022merged,abbe2023sgd}, and the 
\emph{SQ leap}~\cite{joshi2024complexity} complexities. 
Other works have shown that under i.i.d. inputs from shifted product distributions, sparse parities~\cite{malach2021quantifying,daniely2020learning} and juntas~\cite{cornacchia2025low} become significantly easier than under the uniform distribution.



\paragraph{Beyond i.i.d. data.} 
Recent works show that gradient-based learnability of sparse Boolean functions can change beyond the static \emph{i.i.d.} setting. Positive distribution shift (PDS), for example, shows that suitably shifted training distributions can help SGD perform better on the test distribution~\cite{medvedev2026positive}, while curriculum-learning results vary the training distribution over time to expose easier or more informative examples first~\cite{cornacchia2023mathematical,abbe2023provable}. Our results fit into this broader line, but the mechanism is different: the learning advantage does not come from a hand-designed distribution shift, but from the temporal correlations generated by the sampling process itself. Gradient-based optimization with Markovian sampling scheme has also been studied from an optimization perspective~\cite{sun2018markov,even2023stochastic}. These results primarily quantify how temporal dependence affects optimization rates for a fixed objective, whereas our focus is on showing that the temporal correlations can change the information exposed to the gradients, making certain sparse learning problems easier.

\paragraph{Random walk data as an extension of PAC learning.}
\label{sec:related_work_pac}
The problem of learning from correlated examples, particularly through the observation of a Markovian random walk on the Boolean hypercube, has emerged as a rich intermediate model between passive uniform PAC learning and active learning with membership queries. The idea of incorporating temporal differences goes back to Sutton \cite{sutton1988learning}, while the present setting was initially proposed by Aldous and Vazirani~\cite{aldous1990markovian}, and later formalized computationally by Bartlett, Fischer, and H\"offgen~\cite{bartlett2002exploiting}. A notable result in this setting is the result by Bshouty, Mossel, O'Donnell, and Servedio~\cite{bshouty2005learning}, who demonstrated that polynomial-size Disjunctive Normal Form (DNF) formulas can be efficiently learned from a uniform random walk. By bridging the random walk model with statistical query and noise sensitivity frameworks, their work established the viability of random walks to simulate models that rely on uniformly distributed but correlated data. A subsequent work~\cite{elbaz2005separating} provided a rigorous separation, proving that under cryptographic assumptions, random walk learning is strictly more powerful than standard uniform PAC learning.
For the specific problem of learning parities and juntas it was well known that these can be learned using random walks. Arpe and Mossel~\cite{arpe2008agnostically} proved that the class of $k$-juntas can be \emph{agnostically} learned from a uniform random walk on $\{\pm 1\}^d$. Their algorithm, which is combinatorial in nature, achieves this in time polynomial in $d$, $2^{k^2}$, $\varepsilon^{-k}$, and $\log(1/\delta)$. 
This framework was later expanded by Jackson and Wimmer~\cite{jackson2014new}, who generalized the random walk model and corroborated the efficient agnostic learnability of juntas within these broader correlated spaces. More recent work on graphical models also showed the 
advantage of using dynamics such as Glauber dynamics to learn Markov Random Fields~\cite{Bresleretal2018,Gaitondeetal2024}.

\section{Efficient Learning with the Temporal Difference Loss} \label{sec:results}
	In this section, we formalize Theorem \ref{thm:informal} by introducing a stochastic gradient-based algorithm to train a neural network, while leveraging temporal correlations in the data. As we will demonstrate, this algorithm can learn general $k$-juntas which satisfy a natural non-degeneracy condition. We begin by explaining the exact setting and the input to the algorithm, then we present Algorithm \ref{alg:layerwise-sgd}, and state its convergence and statistical guarantees.
	\paragraph{Data.} We consider learning a Boolean $k$-junta, i.e. $f: \{ \pm 1 \}^d \to \bR$ such that $f(x) = \tilde f(x_S)$, where $S \subseteq [d]$ and with $|S|=k = O_d(1)$. We consider inputs generated from a \textit{lazy} random walk on the hypercube. 
	
	\begin{definition}[Lazy random walk on the hypercube with flip probability $p$]\label{def:lazy-rw}
		Let
		\begin{align*}
			&x^{(0)} \sim \Unif \{ \pm  1 \}^d, 
		\end{align*}
		and for each $t \geq 0$, define
		\begin{align*}
			x^{(t+1)} = x^{(t)} - 2 Z_t x_{j_t}^{(t)}e_{j_t},
		\end{align*}
		where $j_t \overset{iid}{\sim} \Unif \{1,...,d\}$, and $Z_t \overset{iid}{\sim} {\rm Ber}(p)$, with some fixed $p\in (0,1)$, and $j_t, Z_t$ are independent of each other, and of $x^{(0)}$. $e_{j_t}\in\{0,1\}^d$ here denotes the $j_t$-th standard basis vector (i.e., $(e_{j_t})_i=\mathds{1}\{i=j_t\}$)
	\end{definition}
	In other words, at each step with probability $1-p$ the chain stays put, and with probability $p$ it flips a uniformly random coordinate. 
	\paragraph{Architecture and algorithm.}
	We consider a two-layered $\ReLU$ network:
	\begin{align}
		\NN(x;\theta) = \sum_{i=1}^N a_i \ReLU(w_i^Tx+b_i), 
	\end{align}
	where for every $i \in [N]$, $w_i \in \mathbb{R}^d$, $a_i,b_i \in \mathbb{R}$, and $\theta =(w_i,a_i,b_i)_{i=1}^N$ denotes the set of all trainable parameters.
   
   To train the network, we will take gradient steps to minimize the \textit{temporal difference (TD) loss}, which is particularly suited to handle temporal correlations, and which we now present. 
    \begin{definition}[Temporal Difference (TD) loss] \label{def:TDloss}
        Define the \emph{temporal difference (TD)} loss with parameter $\alpha \in [0,1]$ as follows. For an estimator $\hat f: \bR^d \to \bR$ relative to a target function $f: \bR^d \to \bR$ computed at samples $x^{(t-1)}$ and $x^{(t)}$ as:
        \begin{align}
		\cL^{\rm TD}_{\alpha}(f, \hat f; x^{(t)},x^{(t-1)}):= \frac{\alpha}{2} \left(\Delta f^{(t)} -  \Delta \hat f^{(t)}\right)^2+ \frac{1-\alpha}{2} \left(f(x^{(t)})- \hat f(x^{(t)}) \right)^2,
	\end{align}
	where $\Delta f^{(t)} = f(x^{(t)}) - f(x^{(t-1)})$ and similarly $\Delta \hat f^{(t)} = \hat f(x^{(t)}) - \hat f(x^{(t-1)})$.
    \end{definition}
    If $\alpha=0$, this reduces to the standard square loss. If $\alpha=1$, the loss depends only on finite differences along the trajectory and is invariant to adding a constant to $\hat f$ (i.e., any $\hat f=f+c$ yields zero loss), so one cannot use $\alpha = 1$ if the goal is to learn the function itself.
    Nevertheless, in the algorithm below, we will start with $\alpha = 1$, since this is useful in order to extract information about the support of $f$. We will then switch to $\alpha < 1$, in order to remove the degeneracy of the loss, and learn the function while knowing its support. This choice is made mainly for convenience in the analysis and should not be viewed as essential; In our experiments in Section~\ref{sec:numerical}, we choose $\alpha$ sufficiently close to $1$, and we keep it fixed throughout training.

    When using the TD loss, we always take mini-batches to consist of consecutive pairs $(x^{(t-1)},x^{(t)})$ taken from a single random-walk trajectory. With these details in place, we implement SGD with the TD loss in Algorithm \ref{alg:layerwise-sgd}, which is carried out in two layer-wise phases. 
    \begin{itemize}
        \item In the first phase, the first layer is trained for $T_1$ steps using a large mini-batch of size $B$, while the second layer is kept fixed. In practice, we show that as long as $B$ is large enough, we can take $T_1 =1$. The main point is that the TD loss picks up information concerning the support of $f$. As a result, the weights at the end of this phase carry a nontrivial signal on the support and little signal outside it. As explained, for convenience, we set $\alpha = 1$ in the TD loss for this step and perform SGD with respect to $\cL_1^{\rm TD}$.

        \item In the second phase, the first layer is kept fixed and only $a$, the weights of the second layer, are trained using $T_2$ steps of SGD. At this point, once the relevant support information has been extracted, the remaining problem is essentially a low-dimensional regression problem. In this phase, it is important to take $\alpha < 1$ to ensure proper learning of the function $f$, and so we set $\alpha = 0$ and perform standard SGD steps for $L^{\rm TD}_0$, which is the $L_2$ loss. To make sure the resulting regression problem is well-posed, we add the regularization term $\frac{\lambda}{2}\|a\|^2$. Remark that, as we explain below, we could remove the regularization if, instead of taking the terminal weight $a^{(T_2)}$, we would instead average the weights along the process.
    \end{itemize} 
    Before the first phase, we initialize for all $i \in [N], j \in [d]$: 
   $w^{(0)}_{ij} = 0,\  a_i^{(0)} = \kappa,\  b_i^{(0)}=1/N,$  
    for some $\kappa>0$. After the first phase of training, we redraw the bias neurons from $b_i \overset{iid}{\sim} \Unif[-A,A]$, for an appropriate $A>0$.
	\begin{algorithm}[t]
		\caption{\textit{Layerwise SGD with temporal difference loss}\\ Init. scale $\kappa>0$, bias range $A$, learning rates $\gamma_1, \gamma_2$, step counts $T_1, T_2$, batch size $B$, and regularization $\lambda > 0$.} 
		\label{alg:layerwise-sgd}
		\begin{algorithmic}[1]
			\Require Training data $\{(x_t, y_t)\}_{t \ge 1}$, model: $\NN(x;(w,a,b)) = \sum_{i=1}^N a_i \sigma(w_i^\top x +b_i)$.
			\State {\bfseries Initialize}: weights $w^{(0)} =0$, $a^{(0)} = \kappa$, and biases $b^{(0)} =1/N$.
			\vspace{0.5em}
			\Statex \textbf{Phase 1: Train the first layer (second layer frozen)}
			\For{$t=0$ {\bfseries to} $T_1-1$ {\bfseries :}} 
			\State Take mini-batch $S^{(t)}=\{(x_s, y_s)\}_{s=1}^B$
			\State Update 
			\(
			w^{(t+1)} \leftarrow w^{(t)} - \gamma_1 \nabla_{w^{(t)}} \cL^{\rm TD}_{1}(S^{(t)};w^{(t)},a^{(0)},b^{(0)})
			\)
			\EndFor
			\vspace{0.5em}
			\Statex \textbf{Phase 2: Train the second layer (first layer frozen)}
			\State Redraw random biases: $\hat b_1,\dots,\hat b_N \sim \Unif[-A,A]^{\otimes N}$
			\For{$t=T_1$ {\bfseries to} $T_2-1$ {\bfseries :}}
            \State Set $S^{(t)} = (x_t,y_t)$
			\State Update $a^{(t+1)} \leftarrow a^{(t)} - \gamma_2 \nabla_{a}\left( \cL^{\rm TD}_{0}(S^{(t)};w^{(T_1)},a^{(t)},\hat b) + \frac{\lambda}{2}\|a^{(t)}\|_2^2\right)$
			\EndFor
			\vspace{0.5em}
			\State \textbf{Output:} Trained model $\NN(x;(w^{(T_1)},a^{(T_2)},\hat b))$
		\end{algorithmic}
	\end{algorithm}

	\paragraph{Main Result}
    To state our main result for Algorithm \ref{alg:layerwise-sgd} we need to introduce one final concept, the influences of $f$. If $f:\{\pm 1\}^d \to \{\pm 1\}$ the influence of $f$ at direction $i$ is $\mathrm{Inf}_i(f) := \mathbb{P}\left(f(x)\neq f(x^{(i)})\right)$, where the probability is over the uniform measure on $\{\pm 1\}^d$ and $x^{(i)}$ is obtained from $x$ by flipping the $i^{\mathrm{th}}$ bit.  More generally, write any $f:\{\pm 1\}^d \to \mathbb{R}$ in the Fourier-Walsh monomial basis $f(x) = \sum_{A \subset [d]} \hat{f}_A x^A$. Then, $\mathrm{Inf}_i(f) = \sum\limits_{i \in A}\hat{f}_A^2$ generalizes the previous definition from Boolean-valued to real-valued functions. In that sense, the influence of $f$ in direction $i$ measures the sensitivity of $f$ to flipping the $i^{\mathrm{th}}$ bit. In order to learn $f$ from the random-walk trajectory, we require that the influences of $f$ are not too small for directions in its support. Thus, if $S$ is the support of $f$, we will henceforth denote $\tau = \min_{i\in S}\mathrm{Inf}_i(f)$. All constants below will implicitly depend on $\tau$ as well as on $\|f\|_{\infty}$.
	\begin{theorem} \label{thm:main_juntas_fd}
		Let $\{(x^{(t)},y^{(t)})\}_{t \in [T]}$, for $T \in \bN$, be such that $\{x^{(t)}\}_{t \in [T]}$ is generated by the $\frac{1}{2}$-lazy random walk on the $d$-dimensional hypercube, as in Def.~\ref{def:lazy-rw}, and $y^{(t)}=f(x^{(t)})$, where $f$ is a $k$-junta, with $k = O_d(1)$. Consider training a 2-layer $\ReLU$ network with $N =\Theta\left(\frac{\log(d)2^kA}{\varepsilon}\right)$ hidden neurons using the layer-wise SGD with the finite difference loss,  Algorithm~\ref{alg:layerwise-sgd}. 
        Then, there exists a constant $C > 0$, which can only depend on $k$, $\|f\|_{\infty}$, and $\min_{i\in S}\Inf_i(f)$, such that for all small enough $\varepsilon, \delta>0$ if we set:
        \begin{itemize}
            \item At initialization, the scale $\kappa =\Theta_d(1)$ and the bias range $A = \Theta_k\left(\frac{\|f\|_{\infty}}{\varepsilon}\right)$.
            \item In Phase I, the batch size  is $B = \Omega_k\left(\frac{d}{\varepsilon^2}\right)$, and the step size is $\gamma_1 = \frac{\sqrt{2Bd}}{\kappa\sqrt{k}}$, with $T_1 = 1$.
            \item In Phase II, the step size is $\gamma_{2} = \tilde{O}_k\left(\frac{\delta^2\varepsilon^4}{d}\right)$ with regularization $\lambda = \Theta_k\left(\delta\varepsilon^2\right)$ and $T_2 = \tilde\Omega_k\left(\frac{d}{\varepsilon^{18}\delta^3}\right)$.
        \end{itemize}
        Then, with probability at least $1 - C\varepsilon$ over the randomness of Phase I and $\hat{b}$, the algorithm will output a network such that:
		\begin{align}
			\E_{x \sim \Unif \{ \pm 1 \}^d} \left[  (\NN(x; \theta^{(T_1+T_2)}) - f(x))^2\right] < \delta.
		\end{align}
	\end{theorem}
    To understand the sample complexity of Algorithm \ref{alg:layerwise-sgd}, we note that Phase I uses a single batch of size $B = \Omega_{k}\left(d\right)$, while Phase II requires $T_2 = \tilde\Omega_k\left(d\right)$ different steps of the random walk. Taken together, and ignoring the dependence on $\varepsilon$ and $\delta$, the algorithm requires running the random walk for $\tilde\Omega_k(d)$ steps. As mentioned above, a simple coupon-collector argument shows that this bound is near optimal. We also note that the dependence on \(k\) is exponential, as is natural for arbitrary \(k\)-juntas: even once the support is known, one must learn an arbitrary Boolean function on \(2^k\) input patterns. This means that, in principle, we could allow $k = c\log(d)$, for a small enough constant $c>0$, and still obtain effective results.
    \paragraph{Ideas of the proof:}
The proof has three main components, corresponding to the two phases of the algorithm and to the non-degeneracy needed to connect them. To understand the first phase, consider the case in which the step from $x_t$ to $x_{t+1}$ flips a coordinate outside the support $S$. In this case, $f(x_{t+1})-f(x_t)=0$, and so the TD loss mostly reacts to moves in relevant coordinates. From this we show that the gradient updates $\nabla_{w^{(t)}}\cL^{\rm TD}_{1}(S^{(t)};w^{(t)},a^{(0)},b^{(0)})$ are essentially supported on $S$. Thus, at the end of Phase $1$, the algorithm has extracted the relevant support information, since $w^{(T_1)}$ is concentrated on the coordinates in $S$.

Given the weights $w^{(T_1)}$ produced after Phase $1$, the second layer is now a linear combination of neurons which essentially depend only on the coordinates in $S$. In light of this, the idea in Phase $2$ is to train the second layer to solve a linear regression problem on the support variables. However, there is an important caveat here. For this regression problem to be well-defined, and for a solution to exist, we must be able to represent $f$ using the weights $w^{(T_1)}$, suitable biases, and the ReLU function. Addressing this point is the most delicate and technical part of the proof, since, as we show, there exist possible choices of $w^{(T_1)}$ which, on one hand, reveal the relevant coordinates, yet cannot represent the function $f$. To circumvent this issue, we prove that these potential bad weights are only encountered with small probability. Thus, we establish an anti-concentration estimate for the batched first-layer update, viewed as an additive functional of the lazy random walk. Our key technical contribution, which may be of independent interest elsewhere, for this problem is a quantitative central limit theorem, based on \cite{kloeckner2019effective}, which allows us to compare this update to a Gaussian and rule out the bad configurations in which the regression problem becomes ill-posed.

The final component of the proof focuses on efficiently solving the resulting linear regression problem. Most results in the literature assume that the samples are drawn i.i.d. from some underlying distribution, whereas in our setting, the samples still follow a random walk. Thus, even after establishing well-posedness, we still need to analyze the complexity of SGD in the presence of temporal correlations. This is done by adapting recent arguments on SGD under Markovian sampling from \cite{even2023stochastic} for a strongly convex target function. This is facilitated by the strongly convex regularization we add in Phase II. Alternatively, we could appeal to the results of \cite{sun2018markov}, which does not require strong convexity, and so does not require the regularization, but only applies to the average weight, rather than the terminal point.

    \section{Noisy SGD Lower Bounds for Standard Losses}
    \label{sec:lower_bound}
    
  In this section, we prove a negative result for SGD with large batch size when training with any convex \emph{pointwise} loss, such as the square loss or the hinge loss. This further highlights the role of the temporal-difference loss in Theorem~\ref{thm:main_juntas_fd}. Here, by a pointwise loss, we mean a loss that evaluates the discrepancy between the predictor and the target on a single example. 
  The lower bound applies to the class of \textit{noisy} gradient-based methods, considered also in previous works~\cite{AS20,malach2021quantifying}, which we formally introduce in Definition~\ref{def:noisy-SGD}.
      \begin{definition}[Noisy-SGD]\label{def:noisy-SGD}
    Let $\NN(\cdot;\theta)$ be a neural network with initialization $\theta^{(0)}$, and let
\(
\{\xi^{(t)}\}_{t\in[T]}
\)
be a training sequence, where each $\xi^{(t)}$ is a sample (or tuple of samples) from which the loss is computed. This formulation covers both standard pointwise losses, where $\xi^{(t)}=(x^{(t)},f(x^{(t)}))$, and temporal-difference losses, where $\xi^{(t)}$ may contain two consecutive samples, e.g.\
\(
\xi^{(t)}=(x^{(t-1)},x^{(t)},f(x^{(t-1)}),f(x^{(t)})).
\)
Let $L(\theta;\xi)$ be an almost surely differentiable loss evaluated at parameter $\theta$ and sample $\xi$, and let $R(\theta)$ be an almost surely differentiable regularizer. For $\lambda\ge 0$, define
\(
L_\lambda(\theta;\xi):=L(\theta;\xi)+\lambda R(\theta).
\) Given learning rates $\{\gamma_t\}_{t\ge 0}$, batch size $B$, gradient range $A>0$, and noise level $\tau\ge 0$, the updates of noisy SGD are defined by
\begin{align}\label{eq:noisy-SGD}
\theta^{(t+1)}
=
\theta^{(t)}
-
\gamma_t\left(
\frac1B\sum_{s=1}^B
\bigl[\nabla_\theta L_{\lambda}(\theta^{(t)};\xi^{(t,s)})\bigr]_A
+
Z^{(t)}
\right),
\end{align}
where for all $t \in \{0,\ldots,T-1\}$, $Z^{t}\sim \cN(0,\tau^2 I_d)$, for some \emph{noise level} $\tau$, and they are independent from other variables, $B$ is the batch size, and $[z]_A : = \argmin_{|y|\leq A} |y-z| $, i.e., whenever the gradient exceeds $A$ (resp. $-A$) it is rounded to $A$ (resp. $-A$). 
    \end{definition}

To establish our result, we extend the lower bound of \cite[Theorem 3]{AS20}, originally proved for i.i.d.\ data, to the case where the training samples are generated by a random walk. Such lower bound holds for \textit{classes of functions}, which leads us to define the orbit of a fixed function. For a Boolean function $f : \{ \pm 1 \}^d \to \bR$, and for a permutation $\pi\in \cS_d$, define the action of $\pi$ on $x\in\{\pm1\}^d$ by
\(
(\pi\cdot x)_i := x_{\pi(i)},\) \(i\in[d].
\)
The \emph{orbit} of $f$ under coordinate permutations is:
\[
\orb(f)
:=
\{\, f\circ (\pi\cdot)\;:\;\pi\in \cS_d \,\}.
\]
\noindent 
We will also use the following notion of cross-predictability, introduced in \cite{AS20}.
\begin{definition}[Cross-Predictability (CP),\cite{AS20}] \label{def:CP}
    Let $P_{\cF}$ be a distribution over functions from $\{ \pm 1 \}^d $ to $\bR$ and $P_{\cX}$ be a distribution over $\{ \pm 1 \}^d$. Then, 
       $ \CP(P_\cX, P_\cF) = \E_{F,F' \sim P_{\cF}} \left[ \E_{x\sim P_{\cX}} [F(x) F'(x)]^2 \right].$
\end{definition}

\noindent 
We may now state our theorem.

\begin{theorem}[Noisy-SGD with large random-walk batches]  
\label{thm:lower-bound}
for $T \in \bN$ let $(x^{(s)}, y^{(s)})_{s \in [m]}$ be such that $\{x^{(s)}\}_{s \in [m]}$ is generated by the $p$-lazy random walk on $\{\pm 1\}^d$ (see Def.~\ref{def:lazy-rw}), and $y^{(s)} =f(x^{(s)})$, for $f: \{ \pm 1 \}^d \to \{\pm 1 \}$ a balanced Boolean function, i.e. such that $\E_{x \sim \cU_d} [f(x)] = o_d(1)$, where $\cU_d := \Unif \{ \pm 1 \}^d$. Consider training a fully connected neural network of size $M$ and i.i.d. initialization with noisy-SGD (see Def.~\ref{def:noisy-SGD}) with any pointwise convex and almost surely differentiable loss, noise level $\tau>0$, gradient range $A$, and batch size $B$. Then, after $T$ steps of training, 
\begin{align} \label{eq:lower_bound}
    \pr\left(\sign(\NN(x;\theta^{(T)}) ) = f(x) \right) \leq \frac 12 + \frac{T \sqrt{M}A}{2\tau} \left( \CP(\cU_d,\orb(f)) + \frac{d}{pB}\right)^{1/4}.
\end{align}

\end{theorem}
The theorem establishes the potential hardness of weak learning, i.e. of obtaining a predictor that achieves accuracy better than a random guess. The statement applies to arbitrary balanced Boolean functions, but it is especially meaningful when the function class is known to be hard to weakly learn, i.e. when the cross-predictability of the orbit is small, such as for parities. In such cases, Theorem~\ref{thm:lower-bound} states that if the batch size is large, so that the right-most factor in~\eqref{eq:lower_bound} is small, then noisy-SGD with bounded gradient precision $A/\tau$ needs either a large number of training iterations $T$, or a network of large size $M$, to perform better than random guess.

\begin{remark}[Large batch size]
The lower bound applies only in the large-batch regime. The intuition is that, when each SGD step averages gradients over many consecutive samples from the random-walk trajectory, the temporal correlations within the batch are largely washed out, provided they decay sufficiently fast. As a result, the averaged gradient becomes close to the one that would be obtained from the stationary distribution of the walk, which in our setting is the uniform distribution on the hypercube. The additional term \(d/(pB)\) measures the residual dependence between samples within a batch drawn from the random walk, which is negligible only if $B$ is large. Therefore, large-batch SGD with a pointwise loss does not significantly exploit temporal correlations, and one recovers essentially the same hardness barrier as in the i.i.d.\ uniform setting.
\end{remark}

Let us now consider the case of parities. Recall that Theorem~\ref{thm:main_juntas_fd} shows that, for every fixed \(k\), a two-layer ReLU network of size \(M=\Theta(d)\), trained by Algorithm \ref{alg:layerwise-sgd} with batch size \(B=\Omega_k(d)\), learns any \(k\)-parity in \(O_k(d)\) steps. In contrast, consider SGD with a standard pointwise loss on the same network of size \(M=\Theta(d)\), with constant gradient precision \(A/\tau=O_d(1)\), constant laziness parameter \(p=\Theta(1)\), and batch size \(B=d^c\), for some $c>0$. Since the orbit of a fixed \(k\)-parity is the class of all \(k\)-parities, \(\CP(\cU_d,P_k)=\binom{d}{k}^{-1}\). Therefore, Theorem~\ref{thm:lower-bound} implies that any such method achieving constant advantage \(\eta>0\) over random guessing must satisfy \(T\ge \Omega_{k,\eta}\left(d^{\frac{\min\{k,c-1\}}{4}-\frac12}\right)\). In particular, if \(\min\{k,c-1\}>6\), then pointwise-loss SGD requires superlinear time to beat random guessing by a constant amount. 

   \begin{figure}[b]
        \centering
        \includegraphics[width=0.49\linewidth]{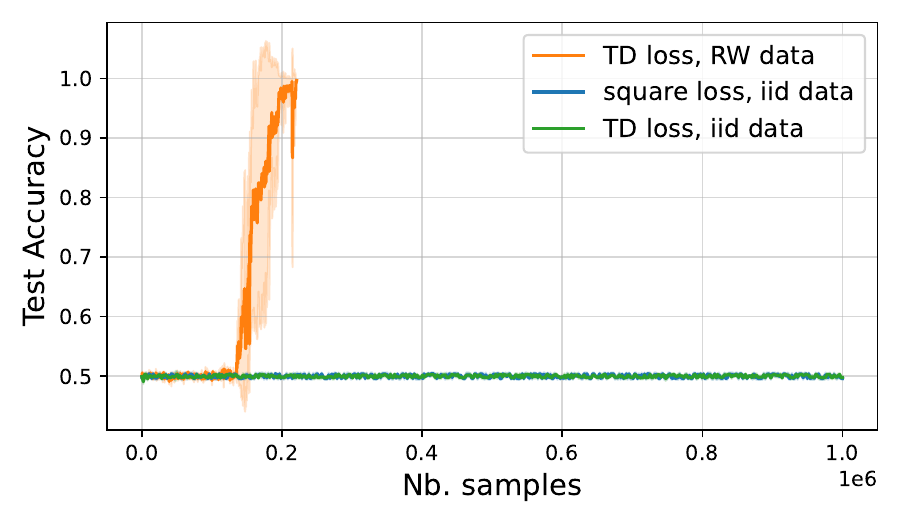}
        \includegraphics[width=0.49\linewidth]{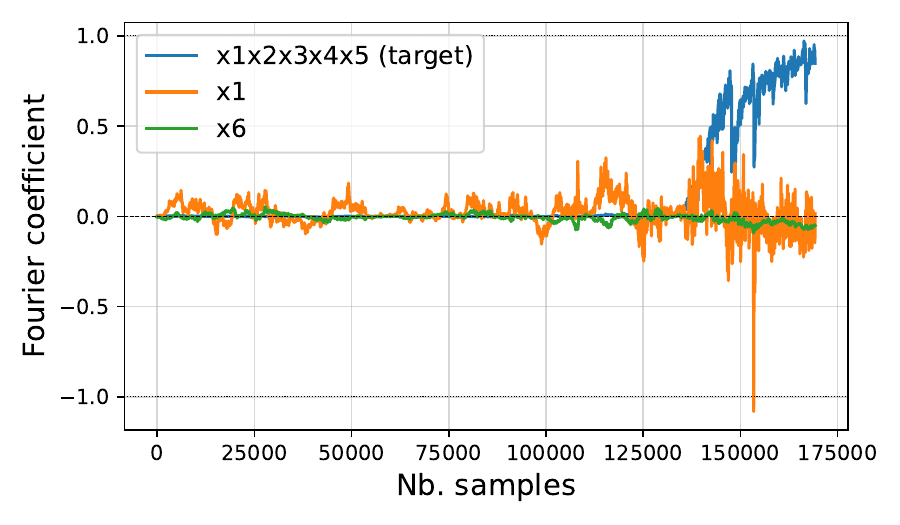}
        \caption{\(5\)-parity with \(d=50\). We train a 4-layer MLP with batch size \(1\), learning rate \(0.005\), TD parameter \(\alpha=0.9\), and random-walk flip probability \(p=0.9\). Left: test accuracy for random-walk data with TD loss, compared with i.i.d. data using either TD loss or square loss. Right: selected Fourier-Walsh coefficients of the learned predictor for random-walk data with TD loss.}
        \label{fig:parity5_TDRW}
    \end{figure}

\section{Numerical Experiments} \label{sec:numerical}
To demonstrate our theoretical results, we ran several experiments in dimension $d=50$. In all experiments we train a 4-layer MLP with $512,1024,64$ hidden neurons and $\ReLU$ activation, by SGD with all layers trained jointly, batch size $1$ and learning rate $0.005$. We use either the temporal difference (TD) loss (Def.~\ref{def:TDloss}) with $\alpha=0.9$ or the standard square loss. The inputs are either sampled online i.i.d. from the uniform distribution on the hypercube, or generated by a lazy random walk with $p = 0.9$. The test loss and accuracy are computed on the standard uniform distribution. We repeat each experiment for $5$ seeds, and we report one standard deviation intervals.

In Figure~\ref{fig:parity5_TDRW}, we consider learning the \(5\)-parity \(f(x)=x_1x_2\cdots x_5\). The left panel reports test accuracy as a function of the number of fresh samples observed during training. With random-walk data and the TD loss, the network reaches high test accuracy with moderate sample complexity, whereas the corresponding i.i.d. baselines fail to learn anything significant. The right panel tracks selected Fourier-Walsh coefficients of the learned predictor during training, for the random-walk experiment with TD loss. Specifically, we plot \(\E_x[\NN(x;\theta^t)x_1]\) in orange, \(\E_x[\NN(x;\theta^t)x_6]\) in green, and the target coefficient \(\E_x[\NN(x;\theta^t)\prod_{i=1}^5 x_i]\) in blue, where expectations are taken over the uniform distribution. The first-order coefficients remain centered around zero, both for a coordinate inside the support (\(x_1\)) and outside the support (\(x_6\)). However, the in-support coefficient (orange) exhibits larger fluctuations than the out-of-support one (green). This variance anisotropy is consistent with our theoretical analysis: the support information is not visible through first-order means, but appears through second-moment effects induced by the random-walk dynamics.

    \begin{figure}[t]
        \centering
        \includegraphics[width=0.49\linewidth]{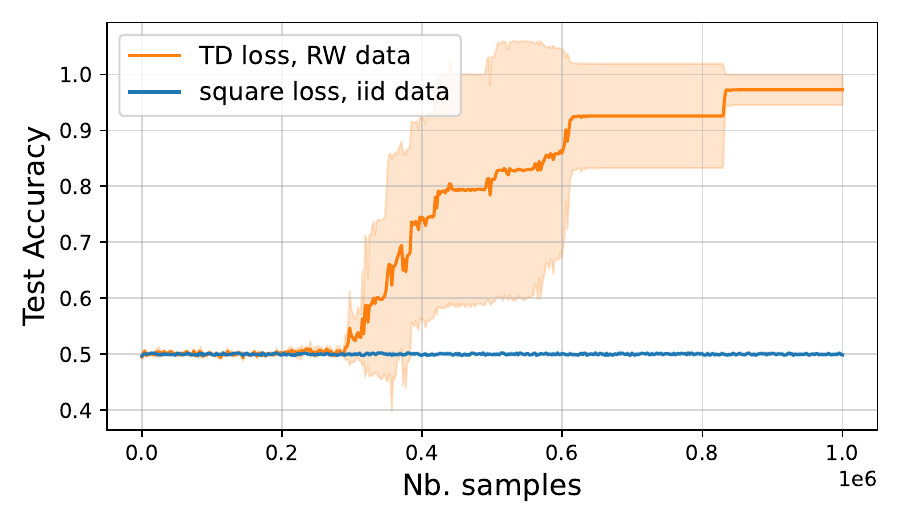}
        \includegraphics[width=0.49\linewidth]{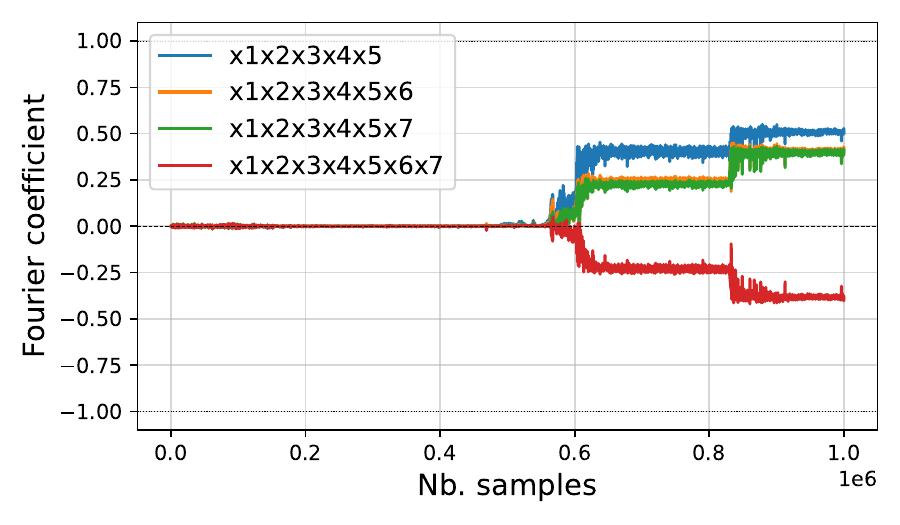}
        \caption{$f(x) = \frac 12 x_1x_2...x_5 (1+x_6+x_7 -x_6x_7) $ with \(d=50\). We train a 4-layer MLP with batch size \(1\), learning rate \(0.005\), TD parameter \(\alpha=0.9\), and random-walk flip probability \(p=0.9\). Left: test accuracy for random-walk data with TD loss, compared with i.i.d. data using square loss. Right: selected Fourier-Walsh coefficients of the learned predictor for random-walk data with TD loss.}
        \label{fig:Junta7}
    \end{figure}

Figure~\ref{fig:Junta7} repeats the same experiment for the \(7\)-junta
\(f(x)=\frac12 x_1x_2\cdots x_5(1+x_6+x_7-x_6x_7)\), with the same qualitative conclusion. In the right panel, we observe that the relevant Fourier-Walsh coefficients of the learned predictor emerge at roughly the same time. This contrasts with the behavior typically seen under i.i.d. uniform sampling, where lower-order terms are learned before higher-order ones~\cite{abbe2023sgd}.

     \begin{figure}[t]
        \centering
        \includegraphics[width=0.49\linewidth]{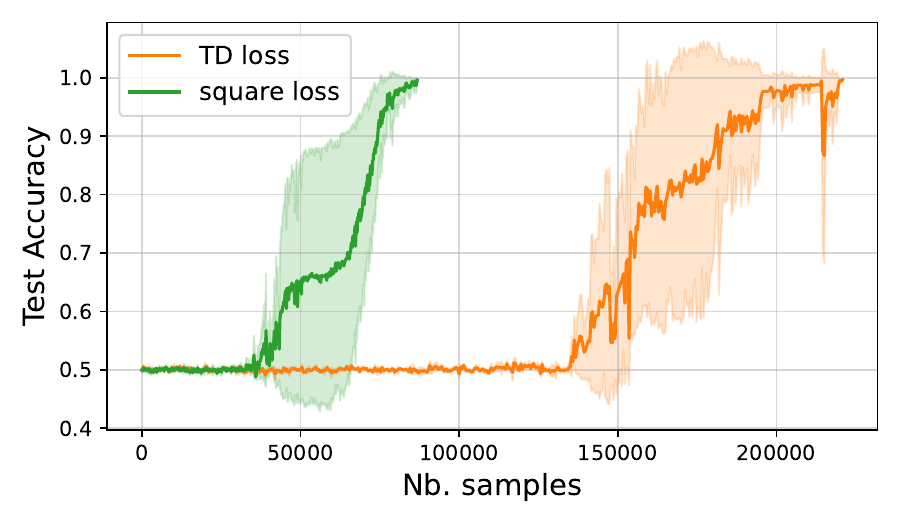}
        \includegraphics[width=0.49\linewidth]{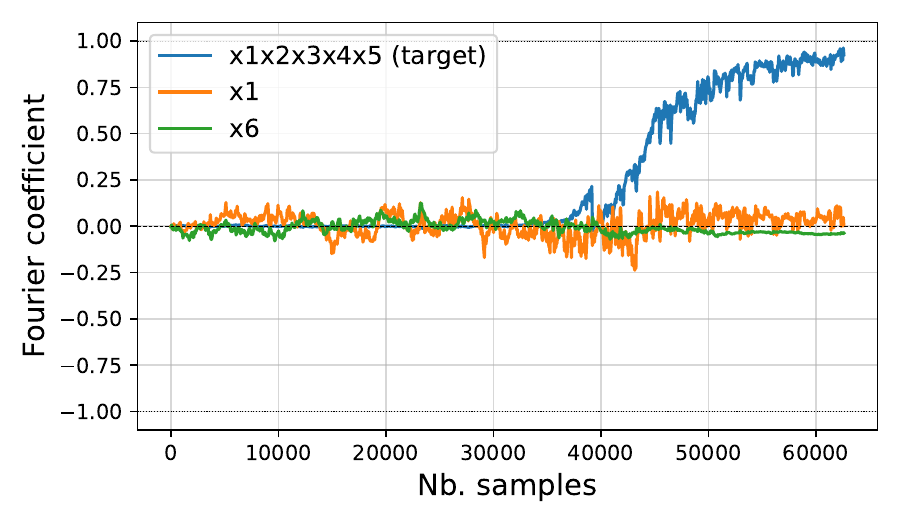}
        \caption{\(5\)-parity with \(d=50\). We train a 4-layer MLP with batch size \(1\), learning rate \(0.005\), and random-walk flip probability \(p=0.9\). Left: test accuracy for random-walk data using TD loss versus square loss. Right: selected Fourier-Walsh coefficients of the learned predictor for random-walk data with square loss.}
        \label{fig:parity5_RW_sqvstd}
    \end{figure}

In Figure~\ref{fig:parity5_RW_sqvstd}, we replace the temporal-difference loss with the standard square loss. Since the experiment uses batch size one, our large-batch lower bound (Theorem~\ref{thm:lower-bound}) does not apply. In this regime, SGD with the square loss successfully learns the parity from random-walk data. This suggests that the small-batch square-loss regime can behave qualitatively differently from the large-batch regime covered by our theory. Interestingly, in these experiments, the square loss alone appears sufficient to exploit the temporal structure, and can even outperform the temporal-difference loss. Understanding this phenomenon theoretically is an interesting direction for future work.



\section{Conclusion}

In this work, we studied how temporal correlations can fundamentally change the complexity of learning some sparse tasks, such as $k$-juntas. Our main positive result shows that, when the data are generated by a lazy random walk on the hypercube, a two-layer ReLU network trained by SGD with a temporal-difference loss can learn Boolean $k$-juntas with linear sample complexity. This shows that the temporal structure of the data can be exploited by SGD to bypass the hardness barriers known in the i.i.d.\ setting. 
On the negative side, we proved that this advantage does not extend to standard convex pointwise losses in the noisy large-batch regime.
Our work leaves several important directions open. First, our lower bound for pointwise losses does not cover the small-batch regime, and our numerical experiments suggest that standard small-batch SGD with square loss may still exploit temporal correlations. Understanding this phenomenon theoretically, and more broadly clarifying how learnability depends on the batch size in the presence of temporal correlations, is an interesting next step. 
Secondly, it would be interesting to extend our results to more general forms of correlated data, beyond single-bit-flip dynamics, and to identify which qualitative properties of temporal correlations influence the learning dynamics.



\section{Acknowledgments}
DM was partially supported by the Brian and Tiffinie Pang Faculty Fellowship. EM was partially supported by 
NSF DMS-2031883, Bush Faculty Fellowship ONR-N00014-20-1-2826 and Simons Investigator award. This work started while EC was visiting MIT for a collaboration visit supported by NSF DMS-2031883.

\bibliography{references}
\bibliographystyle{alpha}






\newpage
\appendix

\section{Proof of Theorem~\ref{thm:main_juntas_fd}}
	Without loss of generality we assume that the support $S=[k]$, i.e. that the target junta is supported on the first $k$ coordinates.
	\subsection{Initial gradient - results of Phase I} 
	The training of the first layer is performed via batched SGD with respect to the temporal difference loss. That is, let $B \in \mathbb{N}$ and let $\{x^{(t)}\}_{t=0}^B$ be the initial segment of the random walk. Define, $\mathcal{L}^{(t)}_\alpha (\theta^{(0)}):= \mathcal{L}_{\alpha}^{TD}\left(f,\NN(.;\theta^{(0)}), x^{(t)}, x^{(t-1)}\right)$. Then,
	\begin{align*}
		w^{(1)} & = w^{(0)} - \frac{\gamma_1}{B}\sum_{t=1}^B \nabla_{w} \mathcal{L}^{(t)}_\alpha  (\theta^{(0)})
        \\
		&= w^{(0)} + G(B),
	\end{align*}
    where we defined:
    \begin{align} \label{def:GB}
        G(B) = - \frac{\gamma_1}{B}  \sum_{t=1}^B \nabla_{w} \mathcal{L}^{(t)}_\alpha  (\theta^{(0)}).
    \end{align}
	Note that by our initialization, $G(B)_{ij}$ does not depend on the hidden unit $i\in [N]$. In the following, we drop the dependence on $i$, to simplify our notation. Note that:
	\begin{align*}
		\nabla (\Delta \NN(.;\theta^{(0)})^{(t)}) &=  a^{(0)} (\mathds{1}((w^{(0)})^T x^{(t)}+ b \geq 0) x^{(t)} -  \mathds{1}((w^{(0)})^T x^{(t-1)}+ b \geq 0) x^{(t-1)} ) \\
		& = \kappa (x^{(t)}- x^{(t-1)}).
	\end{align*}
	Further, since $w^{(0)}_i = 0$, 
	\begin{align*}\forall x  \in \{-1,1\}^d, \quad \NN(x;\theta^{(0)}) = \kappa \sum_{i=1}^N\ReLU(b_i) \implies \Delta \NN(\cdot;\theta^{(0)})^{(t)} = 0.
    \end{align*}
	Recall that to simplify matters, we chose $\alpha=1$ for the first part of training on the first layer, which will nullify the gradient in the irrelevant coordinates. Note that we could have just as well chosen $\alpha$ sufficiently close to $1$ at the cost of complicating the proof. In the case of $\alpha=1$ we then have:
	\begin{align*}
		\nabla_{w} \cL^{(t)}_{1}(\theta^{(0)}) &= -\kappa (x^{(t)}- x^{(t-1)}) \Big( \Delta f^{(t)} -\Delta \NN(.;\theta^{(0)})^{(t)}\Big) \\
		& = - \kappa (x^{(t)}-x^{(t-1)})\Delta f^{(t)} 
	\end{align*}
	Thus, 
	\begin{align} \label{def:GB2}
		G(B)  =  \frac{\kappa \gamma_1}{B} \sum_{t=1}^B (x^{(t)}-x^{(t-1)}) \Delta f^{(t)}.
	\end{align}
    Because $x^{(t)}$ evolves via the lazy hypercube walk (Def.~\ref{def:lazy-rw}), the increment is $x^{(t+1)}-x^{(t)} = -2 Z_{t} x^{(t)}_{j_t}\cdot e_{j_t} $. Note that if $j_t \not\in [k]$, then $f$ does not depend on that coordinate, thus $\Delta f^{(t)} =0$. Therefore, the weights update only when the walk flips a relevant coordinate, and in that case, the update affects only coordinates in the support of the $k$-junta:
	\begin{align} \label{eq:GT-coordinate}
		w_j^{(1)} = (G(B))_{j}=
		\begin{cases}
			\displaystyle
			\frac{\kappa \gamma_1}{B} 
			\sum_{t=1}^B
			\Delta f^{(t)} \big(x^{(t)}_j - x^{(t-1)}_j\big) ,
			& \text{  if } j \in [k], \\[10pt]
			0,
			& \text{  if } j \notin [k].
		\end{cases}
	\end{align}

	\noindent

    \subsubsection{Non-degeneracy of the initial gradient}
    With these choices, it should be rather clear that as long as the batch size $B$ is sufficiently large, $w^{(1)}$ will precisely recover the support of $f$. Thus, let $x \to E(x):=\{\sigma((w^{(1)})^Tx +b_i)\}_{i=1}^N$ be the embedding from $\bR^d$ to $\mathbb{R}^N$ obtained by the first layer of the network. The above observation on the gradient can be rephrased as saying that $E$ induced a map from functions $\{-1,1\}^d$ to functions on $\{-1,1\}^k.$
	Thus, to fully recover $f$ it should be enough to optimize over the second layer of $\NN(\cdot; \theta^{(1)})$, which reduces to a linear regression problem. Concretely, we shall need to find some $a \in \mathbb{R}^N$ such that for $\theta = (w^{(1)},a)$ we have 
    \begin{equation} \label{eq:linearreg}
        \forall x \in \{-1,1\}^d \qquad \NN(x,\theta) = a^TE(x) = f(x).
    \end{equation}
	To make this linear problem well-posed though, we require a non-degeneracy condition on $\NN(\cdot; \theta^{(1)})$. Suppose we have two points $s,r \in \{-1,1\}^k$ such that $\tilde{f}(s)\neq \tilde{f}(r)$ yet $(w^{(1)})^Ts = (w^{(1)})^Tr$. Then, for any choice of $a \in \mathbb{R}^N$ we will have $a^TE(s) = a^TE(r)$, making the above program unfeasible. Our main technical condition for the first stage of the algorithm is hence the following condition on the initial gradient. Recall below that $\tau = \min_{j\in[k]}\mathrm{Inf}_j(f)$ is the minimal influence on the support.
	
		\begin{proposition}[Non-degeneracy of the initial gradient]
		\label{prop:non-degeneracy}
		Assume the data is generated by the lazy random walk of Definition~\ref{def:lazy-rw} with $p \in [0,1/2]$ and that $f:\{\pm1\}^d\to \bR$ is a $k$-junta supported on $[k]$. For $B\ge 1$, let $G(B)=(G_j(B))_{j\in[d]}$ be the gradient defined in~\eqref{eq:GT-coordinate}. Let $\eta = \frac{\kappa \gamma_1}{B}$, so that
		\begin{align*}
		    G(B) = \eta \sum_{t=1}^B \Delta f^{(t)} (x^{(t)}-x^{(t-1)}).
		\end{align*}
		Then there exists a universal constant $C>0$ such that, as long as,
        $$B \geq B_0(\varepsilon) := \frac{C}{\varepsilon^2}\frac{dk^{16}\|f\|_{\infty}^6}{p^7\tau^3} \quad \text{ and } \quad \eta \geq \sqrt{\frac{2d}{Bk}},$$
        for every pair $s,r\in\{\pm 1\}^k$ with $\tilde f(s)\neq \tilde f(r)$, we have
		\begin{equation}
			\mathbb{P}\Bigl(
			\Bigl|\sum_{j=1}^k G_j(B)\,\bigl(s_j - r_j\bigr)\Bigr|
			> \varepsilon
			\Bigr)
			\;\ge\;
			1 -\varepsilon\sqrt{\frac{Ck}{p\tau}},
			\label{eq:non-degeneracy}
		\end{equation}
		where the probability is over the randomness of the random walk.
		\end{proposition}
	
		Our proof of Proposition \ref{prop:non-degeneracy} goes through a quantitative version of the central limit theorem for Markov chains, introduced in \cite{kloeckner2019effective}. To apply this result, we consider the set 
        $$\cV := \{\, v=s - r : s,r \in \{\pm 1\}^k,\ \tilde f(s)\neq \tilde f(r)\,\}
		\subseteq \{0,\pm 2\}^k,$$ and fix $s,r\in\{\pm 1\}^k$ such that $v = s-r \in \cV$.
		Define the scalar projection of the gradient:
		\begin{align}
			\Pi_B(v) := \sum_{j=1}^k G_j(B) v_j.
		\end{align}
		By~\eqref{eq:GT-coordinate},
		\[
		\Pi_B(v) 
		= \eta
		\sum_{u=1}^B \Delta f^{(u)} \sum_{j=1}^k (x^{(u)}_j - x^{(u-1)}_j)v_j
		\]
		Further define the process
		\[
		Y_u :=
		\Delta f^{(u)} \sum_{j=1}^k (x^{(u)}_j - x^{(u-1)}_j)v_j,
		\qquad
		\overline Y_B := \eta\sum_{u=1}^B Y_u,
		\]
		for which
		\begin{equation}
			\Pi_B(v) 
			= \overline Y_B.
			\label{eq:Z-vs-Y}
		\end{equation}
		We rewrite $\overline Y_B$ as an additive functional of simple Markov chain as follows.
		\noindent
		Consider the Markov chain
		\[
		Z_u := (x^{(u-1)},x^{(u)}) \in \Omega := \{\pm 1\}^d \times \{\pm 1\}^d.
		\]
		Since $(x^{(u)})_{u\ge 0}$ is a homogeneous lazy random walk with stationary law $\Unif\{\pm 1\}^d$ and we start from $x^{(0)}\sim\Unif\{\pm 1\}^d$, the chain $(Z_u)$ is an \emph{edge random walk}, stationary with respect to some measure $\mu_0$ on $\Omega$.
		Define the observable
		\[
		\varphi_{v}(x,y)
		:= \bigl(f(y)-f(x)\bigr)\cdot \sum_{j=1}^k (y_j-x_j) v_j,
		\qquad (x,y)\in\Omega.
		\]
		Then $Y_u = \varphi_{v}(Z_u)$ and $\overline Y_B = \eta\sum_{u=1}^B \varphi_{v}(Z_u)$. To avoid having to depend on the ambient dimension $d$, let us note that 
        $$\mathrm{support}(x^{(u)}-x^{(u-1)}) \cap [k] = \emptyset \implies \varphi_v(Z_u)=0.$$ 
        We can thus couple $\overline{Y}_B$ with an additive function of a $k$-dimensional random walk as follows:
        Let $Z'_u$ be an edge random walk on $\{\pm 1\}^k \times \{\pm 1\}^k,$ and set $Y'_u = \varphi_v(Z'_u)$. By the above observation $Y'_u = Y_u$ conditioned on $\mathrm{support}(x^{(u)}-x^{(u-1)}) \cap [k] \neq \emptyset$ while also accounting for lazy steps in this direction, and so $Y'_u$ is well-defined.
        For $B \in \mathbb{N}$ we are thus led to consider the normalized lower-dimensional function
        $$\overline{Y}'_{B} = \frac{1}{\sqrt{B}}\sum\limits_{u=1}^{B}Y'_u.$$
        Further, recall from Definition \ref{def:lazy-rw} that $j_u$ is the proposed flip of the random walk at step $u$, and define the random variable
        $$B_k := \#\{0\leq u \leq B \mid  j_u \in [k]\}.$$
        That is, the number of steps in which $Z_u$ evolved along the first $k$ coordinates, where we also include the lazy steps that did not materialize to a flip. With respect to these counts, we now have the following equality in law
        $$\overline{Y}_B = \eta\sqrt{B_k}\overline{Y}'_{B_k}.$$
		Let
		\[
		m_{v} := \mathbb{E}[\varphi_v(Z'_1)],
		\qquad
		\sigma_{v}^2 := \sigma^2(\varphi_{v}),
		\]
		where $\sigma^2(\cdot)$ is the dynamical variance, given by
		$$\sigma_v^2 = \lim_{T\to \infty}\mathrm{Var}(\overline Y'_{T}).$$
		The proof of Proposition \ref{prop:non-degeneracy} crucially relies on the following fact, that for $B$ large enough (and hence also for $B_k$ large enough) $\overline Y'_{B_k}$ is essentially Gaussian.
		\begin{lemma} \label{lem:BEedges}
				Let $G$ be a standard Gaussian variable. There exists a universal constant $C>0$ such that for any $v \in \cV$ and $x \in \mathbb{R}$,
					$$\sup\limits_{x\in \mathbb{R}}\left|\mathbb{P}\left(\sigma_v^{-1}\left(\overline Y'_T - \E[\overline Y'_T]\right) \leq x\right) - \mathbb{P}\left(G \leq x\right)\right| \leq C\frac{k^7\|f\|_{\infty}^3}{p^2\min(\sigma_v^{3},\sigma_v)\sqrt{T}}.$$
		\end{lemma}
		We will prove Lemma \ref{lem:BEedges} in Section \ref{sec:ACmarkov} below. For now, to properly apply the Lemma, we will need to estimate the first moments of $\overline Y'_T$.
	
		\begin{lemma} \label{lem:Yexpec}
		Let $Y'_u$ be as above. Then,
		$$\mathbb{E}[Y'_u] =  \frac{4p}{k}\sum\limits_{j=1}^k v_j\hat{f}_j,$$
		where $\hat{f}_j$ are the linear Fourier coefficients of $f$. In particular $\E[Y'_u] = O(p)$.
	\end{lemma}
	\begin{proof}
		Let $Z' = (u,v)$ be drawn according to invariant measure, that is $(u,v)$ is a random edge of $\{\pm 1\}^k \times \{\pm 1\}^k$.
		We have,
		$$\mathbb{E}\left[\varphi_v(Z')\right] = \frac{2p}{k}\sum\limits_{j=1}^k\mathbb{E}\left[(f(u) - f(u^{(j)}))u_j\right]v_j =  \frac{4p}{k}\sum\limits_{j=1}^k\mathbb{E}\left[D_j(v_jf(u))u_j\right],$$
		where $u$ is uniform and $u^{(j)}$ is obtained from $u$ by flipping the $j$'th coordinate.
		Integrating by parts on the discrete cube, we see
		$$\sum\limits_{j=1}^n\mathbb{E}\left[D_j(v_jf(u))u_j\right] =\sum\limits_{j=1}^d\mathbb{E}\left[v_jf(u)u_j\right]= \sum\limits_{j=1}^d v_j\hat{f}_j,$$
	\end{proof}
	\begin{lemma}\label{lem:Yvar}
		Suppose that $p\leq \frac{1}{2}$. Then, for any $v$, it holds that
        $$\mathrm{Var}(\varphi_v(Z'_1)) \leq \sigma^2_v \leq 2\frac{k}{p}\mathrm{Var}(\varphi_v(Z'_1)).$$
	\end{lemma}
	\begin{proof}
		We have that 
		\begin{align*}
			\sigma^2_v = \lim\limits_{T\to \infty}\frac{1}{T}\mathrm{Var}\left(\sum\limits_{u=1}^T\varphi_v(Z'_u)\right) &= \mathrm{Var}(\varphi(Z'_1)) + 2\sum\limits_{u=1}^{\infty}\mathrm{Cov}(\varphi_v(Z'_1),\varphi_v(Z'_u))\\
			&= \mathrm{Var}(\varphi(Z'_1))  + 2\sum\limits_{u=1}^{\infty}\mathrm{Cov}(\varphi_v(Z'_1),P_u\varphi_v(Z'_1)).
		\end{align*}
        We shall start with the lower bound, which is the more difficult part.
        We claim that for any $u \geq 1$, $\mathrm{Cov}(\varphi_v(Z'_1)\varphi_v(Z'_u)) \geq 0$, the bound will immediately follow.
        Indeed, for $z'$ and edge write $z' = (x,x')$, and similarly for the random walk $Z_u = (X_{u-1}, X_u)$. We first note that
        $$h_v(x,x') := P_1\varphi_v(z) = \sum\limits_{x''}\mathbb{P}\left(X_2 = x''| X_1 = x'\right) \varphi_v(x', x'').$$
        In particular, $h_v$ is only a function of $x'$ and is independent from $x$, and so we will instead write $h_v(x')$. Moreover, it is immediate to verify that $\varphi_v(x,x') = \varphi_v(x',x)$, and so 
        $$\mathbb{E}\left[\varphi_v(Z'_1)|X_1\right] = \mathbb{E}\left[\varphi_v(X_0,X_1)|X_1\right] = \mathbb{E}\left[\varphi_v(X_0,X_1)|X_0\right],$$
        i.e. conditioning on the past and on the present has the same effect.
        Now, for $u\geq 0$, since $X_u$ is Markovian
        $$\mathrm{Cov}(\varphi_v(Z'_1),\varphi_v(Z_u')) = \mathrm{Cov}(\E\left[\varphi_v(X_0,X_1)|X_1\right],\E\left[\varphi_v(X_{u-1},X_u)|X_1\right]). $$
Because of the symmetry, as above, we have $\E\left[\varphi_v(X_0,X_1)|X_1\right] = h(X_1)$. Furthermore, by definition,
$$\E\left[\varphi_v(X_{u-1},X_u)|X_1\right] = P_u(X_0,X_1) = P_{u-1}h(X_1).$$
Combining, we get
$$\mathrm{Cov}(\varphi_v(Z'_1),\varphi_v(Z_u')) = \mathrm{Cov}(h(X_1),P_{u-1}h(X_1)) \geq 0.$$
In the inequality, we have used that on functions that only depend on $X_1$, $P_{u-1}$ identifies with the usual semigroup of the lazy random walk on $\{\pm1\}^k$. Since we take $p \leq \frac{1}{2}$, it is positive semi-definite, which leads to the inequality.
        
        For the upper bound, let $\tau$ be the spectral gap of $Z'_t$, so that $\mathrm{Var}\left(P_u \varphi_v(Z'_1)\right) \leq (1-\tau)\mathrm{Var}\left(\varphi_v(Z'_1)\right)$. 
        Thus, with the Cauchy Schwartz inequality, we get,
		\begin{align*}
			\left|\mathrm{Cov}(\varphi_v(Z'_1)\varphi_v(Z'_u))\right| \leq (1-\tau)^u\mathrm{Var}(\varphi_v(Z'_1)) .
		\end{align*}
		Since $\tau \geq \frac{p}{k}$, we obtain the result by summing the geometric series.
	\end{proof}
	We can now prove Proposition \ref{prop:non-degeneracy}.
	\begin{proof}[Proof of Proposition \ref{prop:non-degeneracy}]
		With Lemma \ref{lem:BEedges} and Lemma \ref{lem:Yvar} we get
        \begin{equation} \label{eq:MarkovCLT}
            \sup\limits_{x\in \mathbb{R}}\left|\mathbb{P}\left(\sigma_v^{-1}\left(\overline Y'_T - \E[\overline Y'_T]\right)\leq x\right) - \mathbb{P}\left(G \leq x\right)\right| \leq C\frac{k^{7}}{p^{2}}\frac{\|f\|_{\infty}^3}{\mathrm{Var}(\varphi_v(Z'_1))^\frac{3}{2}\sqrt{T}}.
        \end{equation}
        We begin by bounding $\mathrm{Var}(\varphi_v(Z'_1))$ from below by decomposing the functional as $\varphi_v = \sum\limits_{j=1}^k\varphi_{v,j}$, where for each $i \in [k]$, 
    $$\varphi_{v,j}(x,y) = (f(x)-f(y))(x_j-y_j)v_j,$$
    with $\varphi_{v,j}(Z'_1)\varphi_{v,j'}(Z'_1) = 0$ almost surely whenever $j \neq j'$. Thus, since $v_j \in \{-2,0,2\}$
    \begin{align*}
        \mathbb{E}\left[\varphi^2_v(Z'_1)\right] = \sum\limits_{j=1}^k \mathbb{E}\left[\varphi^2_{v,j}(Z'_1)\right]
        = \frac{16p}{k}\sum\limits_{v_j \neq 0} \mathbb{E}\left[(f(X)- f(X^{(j)}))^2\right],
    \end{align*}
    where $X$ is uniform on $\{\pm1\}^k$ and $X^{(j)}$ is obtained from $X$ by flipping the $j$'th coordinate. 
    Observe that 
    $$\mathbb{E}\left[(f(X)- f(X^{(j)}))^2\right] = 4\mathrm{Inf}_j(f) = 4\sum\limits_{j \in S}\hat{f}^2_S,$$
    and 
    $$\mathbb{E}\left[\varphi^2_v(Z'_1)\right] = \frac{64p}{k}\sum\limits_{v_j \neq 0}\mathrm{Inf}_j(f).$$
    On the other hand, by Lemma \ref{lem:Yexpec} we have
    \begin{align*}
        \mathbb{E}\left[\varphi_v(Z'_1)\right]^2 = \frac{16p^2}{k^2}\left(\sum\limits_{j=1}^k\hat{f}_jv_j\right)^2 \leq \frac{16p^2}{k^2} \sum\limits_{j=1}^k v_j^2\sum\limits_{v_j \neq 0} \hat f^2_j \leq \frac{64p^2}{k} \sum\limits_{v_j \neq 0}\mathrm{Inf}_j(f),
    \end{align*}
    where we have used the Cauchy-Schwartz inequality and the fact that $v_j \in \{-2,0,2\}$. Combining, and using the fact that $p \leq \frac{1}{2}$, we obtain
    $$\mathrm{Var}(\varphi_v(Z'_1) \geq \frac{32p}{k}\sum\limits_{v_j\neq 0}\mathrm{Inf}_j(f) \geq \frac{32p}{k}\min_{j\in[k]}\mathrm{Inf}_j(f).$$
    Plugging his into \eqref{eq:MarkovCLT}, we get,
     \begin{equation} 
            \sup\limits_{x\in \mathbb{R}}\left|\mathbb{P}\left(\sigma_v^{-1}\left(\overline Y'_T - \E[\overline Y'_T]\right)\leq x\right) - \mathbb{P}\left(G \leq x\right)\right| \leq C'\frac{k^{\frac{17}{2}}}{p^{\frac{7}{2}}}\frac{\|f\|^3_{\infty}}{\min\limits_{j\in[k]}\mathrm{Inf}_j(f)^\frac{3}{2}\sqrt{T}}.
        \end{equation}
        Recall that for $B \in \mathbb{N}$, we've defined the random variable 
        $$B_k := \#\{0\leq u \leq B \mid  j_t \in [k]\},$$
        counting the number of potential flips for the first $k$ coordinates.
        Consider the event $A := \{\frac{k}{2d}\leq \frac{B_k}{B} \leq \frac{3k}{2d}\}$.
        As a sum of independent indicator variables, by Chernoff's inequality, we have
        $$\mathbb{P}\left(A\right) \geq 1 - 2e^{-\frac{k}{12d}B}.$$
        Further note that under $A$ we can invoke \eqref{eq:MarkovCLT} for $T = B_k$ and choose $B$ to force the error bound to be of order $\varepsilon$.
		Thus, as long as $B \geq B_0(\varepsilon) = 4\frac{C^2}{\varepsilon^2}\frac{dk^{16}}{p^7}$, since the Gaussian density is bounded uniformly by $\frac{1}{\sqrt{2\pi}}$,
		we get 
		 \begin{align*}
		     \mathbb{P}\Bigl(
		 \Bigl|\sum_{j=1}^k G_j(B)\,\bigl(s_j - r_j\bigr)\Bigr|
		\leq \varepsilon
		 \Bigr) &= \mathbb{P}\left(|\overline Y_B| \leq \varepsilon\right)  = \mathbb{P}\left(\eta\sqrt{B_k}\sigma_v^{-1}|\overline Y_{B_k}| \leq \sigma_v^{-1}\varepsilon\right)\\
         &\leq\mathbb{P}\left(\eta\sqrt{B_k}\sigma_v^{-1}|\overline Y_{B_k}| \leq \sigma_v^{-1}\varepsilon\Big|A\right) + \mathbb{P}(A^c) \\
         &\leq 10\sigma^{-1}_v\varepsilon \leq 10\varepsilon\sqrt{\frac{k}{p\min\mathrm{Inf}_j(f)}}
		 \end{align*}
         where in the inequality before last we have used that $\eta \geq \sqrt{\frac{2d}{Bk}}$, and so under $A$, $\frac{1}{\eta\sqrt{B_k}}\leq 1.$
	\end{proof}

    \subsection{Markov chain linear regression - results of Phase II }
    Given the non-degeneracy promised by Proposition \ref{prop:non-degeneracy}, the main idea is that the second phase of Algorithm \ref{alg:layerwise-sgd} now solves a well-posed linear regression problem. The main technical detail that requires accounting for is that the samples are not i.i.d., and so we need to appeal to results concerning Markov chains. We shall first prove that there is indeed a solution to the linear regression presented in \eqref{eq:linearreg}, and then show that SGD performed on the random walk can find this solution.
    \subsubsection{Certificate}
    As explained, we begin by showing that Proposition \ref{prop:non-degeneracy} implies the existence of a solution to \eqref{eq:linearreg}.
	\begin{lemma}[Certificate] \label{lem:certificate}
        Under the same conditions of Proposition \ref{prop:non-degeneracy}, let $b_i \overset{iid}{\sim} \Unif [-A,A]$, with $A \geq\|w_{[k]}^{(1)}\|_{1}+\varepsilon$ and $i \in [N]$, with $N = \Omega(\frac{\log(d)A}{\varepsilon k})$. Then, there exists a constant $C = C(k,\min\limits_{i \in[k]}\mathrm{Inf}_i(f))$, such that with probability at least $1-C\varepsilon$, there exists $a^* \in \bR^N$, with $ \|a^* \|_\infty =O(\|f\|_{\infty}/\varepsilon) $ such that 
		\begin{align} \label{eq:astar}
			f(x) = \sum_{i=1}^N a_i^* \ReLU(\langle w_i^{(1)}, x\rangle + b_i),
		\end{align}
		for all $x \in \{ \pm 1 \}^d$.
	\end{lemma}

	\begin{proof}
		Fix $\varepsilon>0$. By applying the union bound to Proposition~\ref{prop:non-degeneracy}, there exists a constant $C = C(k,\min\limits_{i \in[k]}\mathrm{Inf}_i(f))$ such that with probability at least $1-C\varepsilon$, $w^{(1)}$ is such that, for every $s,t \in \{\pm 1\}^k$ with $\tilde{f}(s) \neq \tilde{f}(t)$:
		\begin{align} \label{eq:seper}
			\Big| \sum_{j=1}^k w_j^{(1)} (s_j-t_j)  \Big| > \varepsilon.
		\end{align}
		Define the projected values
		\begin{align}
			v(s) := \sum_{j=1}^k s_j w_j^{(1)}, \qquad s \in \{ \pm 1 \}^k.
		\end{align}
		Note that $v(s)$ takes at most $M \leq 2^k$ different values, which we order as $U= \{ u_1 < ...< u_M\}$.  
        For these values \eqref{eq:seper} implies
        $$\min_{m \in [M-1]} (u_{m+1}-u_m) > \varepsilon,$$
        unless $u_{m+1}$ and $u_{m}$ corresponds to two values $s,t \in \{\pm\}^k$ for which $\tilde{f}(s) = \tilde{f}(t)$, and so we will not distinguish between such values.
        Consider the interval $I= [-u_{M} -\varepsilon, - u_1 +\varepsilon ] $. Note that $I \subseteq [-A,A]$, with $A = \|w_{[k]}^{(1)}\|_{1}+\varepsilon$. Further consider the following $2M$ disjoint sub-intervals of $I$, each one of width $\varepsilon/8$: $I^{\pm}_m:= [-u_m\pm\varepsilon/8, -u_m\pm\varepsilon/4]$. Now, if $N$ biases are sampled uniformly from $[-A,A]$, for each $j\in [M]$, the probability that none of the $N$ samples falls into $I_j$ is 
		\begin{align*}
			\left( 1- \frac{\varepsilon/8}{2A} \right)^{N} \leq \exp\left(- \frac{\varepsilon}{16A}N \right).
		\end{align*}
		By a union bound, and provided that $N \geq \frac{C' \log(d) A}{\varepsilon }$, the probability that every interval $I_j$ is hit at least once is bounded from below by
		\begin{align*}
			1- M \exp\left(- \frac{\varepsilon}{2A}N \right) \geq 1-d^{-C'/2}.
		\end{align*}
		For each $j \in [M]$, choose one $b_{h_j} $ such that $b^{\pm}_{h_j}  \in I^{\pm}_j$, and denote by $J = \{ b^{\pm}_j:j\in [M]\} \subset [n]$ such a set of representatives. We will show that the set $-J$ consists of the promised biases.
        Consider the linear combination of neurons
        $$H_j(x) = \frac{\ReLU(\langle w^{(1)},x\rangle- b_j^{-}) - \ReLU( \langle w^{(1)},x\rangle- b_j^{+})}{b_j^{+} - b_j^{-}}.$$
        Since $b_j^- \leq  b_j^+$, the following conditions are easy to verify: $H_j(x) = 1$ when $\langle w^{(1)},x\rangle > b^+_j$ and $H_j(x) = 0$ when $\langle w^{(1)},x\rangle < b^-_j$.
        It follows that there exists a sequence of numbers $c_j$ such that for every $x \in \{\pm 1\}^d$,
        $$f(x) = \sum\limits_{j=1}^M (c_{j} - c_{j-1})H_j(x).$$
        Specifically, we can build $c_j$, starting from $c_0 = 0$, to satisfy $c_j H_j(x) = f(x)$. Note that the choice of the set $U$ makes this construction 
        well-defined. This construction thus constitutes the network promised in the statement of the theorem.
        Moreover, the vector $a^*$ has components of the form 
        $$\frac{c_j-{c_{j-1}}}{b_j^+ - b_j^-} \leq 10\frac{\|f\|_{\infty}}{\varepsilon},$$
        where we have used that $b_j^+ - b_j^{-} \geq \frac{\varepsilon}{4}$.
        
	\end{proof}
	

\subsubsection{Second Layer Training}
We now analyze Phase II using the strongly convex Markov-chain SGD bound of \cite[Cor.~1]{even2023stochastic}. We shall first explain how to set up the problem as a strongly convex and smooth linear regression problem and bound its condition number.
Since $w^{(1)}_j=0$ for all $j>k$, both $f$ and the feature map induced by the frozen first layer depend only on the first $k$ coordinates. It is therefore enough to work with the projected chain,
\[
v^{(t)}:=x^{(t)}_{[k]} \in \{\pm 1\}^k .
\]
Let $\hat{b} \in \bR^N$ be a vector of biases with coordinates i.i.d. from $\mathrm{Unif}[-A,A]$, as in Lemma \ref{lem:certificate}. Using these biases, for $x\in \{\pm 1\}^k$, define the feature map
\[
\Phi(x):= \Bigl(\ReLU(\langle w^{(1)},x\rangle+\hat b_i)\Bigr)_{i=1}^N \in \bR^N,
\qquad
R_\Phi:= \max_{x\in\{\pm1\}^k}\|\Phi(x)\|_2.
\]
Then
\[
\NN(x;(w^{(1)},a,\hat b)) = \sum\limits_{i=1}^N a_i \ReLU\left(\langle w^{(1)},x\rangle + \hat{b}_i\right) = a^\top \Phi(x_{[k]}).
\]
Moreover, $(v^{(t)})_{t\ge 0}$ is a Markov chain on $\{\pm1\}^k$ with transition kernel
\[
P_k(s,s')=
\begin{cases}
p/d, & \text{if } s'=s^{(j)} \text{ for some } j\in[k],\\
1-kp/d, & \text{if } s'=s,\\
0, & \text{otherwise,}
\end{cases}
\]
where $s^{(j)}$ denotes $s$ with the $j^{\mathrm{th}}$ coordinate flipped.
This chain is reversible with stationary law $\pi=\Unif(\{\pm1\}^k)$, and its characters are eigenfunctions with eigenvalues
\[
\lambda_S = 1-\frac{2p|S|}{d}, \qquad S\subseteq[k].
\]
Hence the spectral gap is $2p/d$. 
We denote by
\[
\tau_{\rm mix}
:=
\inf\Bigl\{
t\ge 1:\ \forall \nu,\ d_{\rm TV}(\nu P^t,\pi)\le 1/4
\Bigr\},
\]
the total-variation mixing time of the sampling chain, where $P$ is the transition kernel and $\pi$ its stationary distribution. In our case, 
\begin{equation}\label{eq:phaseII-mixing}
\tau_{\rm mix} \le \frac{d}{2p}\log(2^{2k+1}) = O\!\left(\frac{dk}{p}\right).
\end{equation}
For $\lambda>0$, define the regularized per-state loss
\[
\ell_\lambda(s;a):=
\frac12\bigl(a^\top \Phi(s)-\tilde f(s)\bigr)^2
+\frac{\lambda}{2}\|a\|_2^2,
\]
and the total loss
\[
L_\lambda(a):=
2^{-k}\sum_{s\in\{\pm1\}^k}\ell_\lambda(s;a).
\]
Let $a_\lambda^\star=\arg\min L_\lambda(a)$. For each $s\in\{\pm1\}^k$,
\[
\nabla_a \ell_\lambda(s;a)
=
\bigl(a^\top \Phi(s)-\tilde f(s)\bigr)\Phi(s)+\lambda a,
\qquad
\nabla_a^2 \ell_\lambda(s;a)
=
\Phi(s)\Phi(s)^\top+\lambda I.
\]
Therefore $\ell_\lambda(s;\cdot)$ is $\lambda$-strongly convex and $M_\lambda$-smooth, with 
$M_\lambda:=R_\Phi^2+\lambda$ and condition number $
\kappa_\lambda:=\frac{M_\lambda}{\lambda}.
$
Let $a^*$ be the vector promised by Lemma \ref{lem:certificate}, so that
\[
a^{*\top}\Phi(s)=\tilde f(s)
\qquad
\forall s\in\{\pm1\}^k.
\]
By optimality of $a_\lambda^\star$,
\begin{equation}\label{eq:ridge-comparison}
L_\lambda(a_\lambda^\star)
\le
L_\lambda(a^*)
=
\frac{\lambda}{2}\|a^*\|_2^2.
\end{equation}
In particular, since both terms are positive:
\begin{equation}\label{eq:ridge-bias-bound}
2^{-k}\sum_{s\in\{\pm1\}^k}\bigl(a_\lambda^{\star\top}\Phi(s)-\tilde f(s)\bigr)^2
\le
\lambda \|a^*\|_2^2,
\qquad \text{ and, } \qquad 
\|a_\lambda^\star\|_2\le \|a^*\|_2.
\end{equation}
Define
\[
\sigma_{\lambda,\star}^2
:=
\max_{s\in\{\pm1\}^k}\|\nabla_a \ell_\lambda(s;a_\lambda^\star)\|_2^2.
\]
Then, using \eqref{eq:ridge-bias-bound},
\begin{align*}
\max_{s\in\{\pm1\}^k}\bigl|a_\lambda^{\star\top}\Phi(s)-\tilde f(s)\bigr|
&\le
2^{k/2}
\left(
2^{-k}\sum_{u\in\{\pm1\}^k}
\bigl(a_\lambda^{\star\top}\Phi(u)-\tilde f(u)\bigr)^2
\right)^{1/2}
\nonumber\\
&\le 2^{k/2}\sqrt{\lambda}\,\|a^*\|_2.
\end{align*}
Hence
\begin{equation}\label{eq:sigma-lambda-star}
\sigma_{\lambda,\star}^2
\le
\Bigl(2^{k+1}R_\Phi^2\lambda + 2\lambda^2\Bigr)\|a^*\|_2^2
\le
\bigl(2^{k+1}+2\bigr)M_\lambda\lambda \|a^*\|_2^2.
\end{equation}
Finally, note that
\begin{equation}\label{eq:Rphi-bound}
R_\Phi \le \sqrt{N}\bigl(\|w^{(1)}\|_1+A\bigr),
\qquad
M_\lambda \le N\bigl(\|w^{(1)}\|_1+A\bigr)^2+\lambda.
\end{equation}

\begin{proposition}[Phase II: convergence of the second layer]\label{prop:phaseII-even}
Assume the event of Proposition~\ref{prop:non-degeneracy}. Fix $\lambda>0$ and run Phase II with the update
\[
a^{(t+1)}
=
a^{(t)}
-
\gamma_2 \nabla_a \ell_\lambda(v^{(t)};a^{(t)}),
\]
with step size
\begin{align*}
 \gamma_2 =
\min\!\left\{
\frac{1}{M_\lambda},
\;
\frac{1}{\lambda T_2}
\log\!\left(
\frac{\lambda^2 T_2 \,\|a^{(0)}-a_\lambda^\star\|_2^2}
{M_\lambda\,\tau_{\rm mix}\,\sigma_{\lambda,\star}^2}
\right)
\right\}.
\end{align*}
Then, 
\begin{equation}\label{eq:even-distance}
\E\bigl[\|a^{(T_2)}-a_\lambda^\star\|_2^2\bigr]
\le
2e^{-T_2/\kappa_\lambda}\|a^{(0)}-a_\lambda^\star\|_2^2
+
\widetilde{\mathcal O}\!\left(
\frac{M_\lambda \tau_{\rm mix}\sigma_{\lambda,\star}^2}{\lambda^3 T_2}
\right).
\end{equation}
Consequently,
\begin{align}
&\E_{x\sim\Unif(\{\pm1\}^d)}
\Bigl[
\bigl(\NN(x;(w^{(1)},a^{(T_2)},\hat b))-f(x)\bigr)^2
\Bigr]
\nonumber\\
&\qquad\le
\lambda\|a^*\|_2^2
+
2L_\lambda e^{-T_2/\kappa_\lambda}\|a^{(0)}-a_\lambda^\star\|_2^2
+
\widetilde{\mathcal O}\!\left(
\frac{L_\lambda^3\tau_{\rm mix}\|a^*\|_2^2}{\lambda^2 T_2}
\right).
\label{eq:phaseII-pop-risk}
\end{align}
\end{proposition}

\begin{proof}
Apply Corollary~1 of \cite{even2023stochastic} to the finite family of functions
\[
\bigl\{\ell_\lambda(s;\cdot): s\in\{\pm1\}^k\bigr\},
\]
using the facts established above: each $\ell_\lambda(s;\cdot)$ is $\lambda$-strongly convex and $L_\lambda$-smooth, the sampling chain is $(v^{(t)})_{t\ge 0}$, and $\sigma_{\lambda,\star}$ is bounded by \eqref{eq:sigma-lambda-star}. This gives \eqref{eq:even-distance}.

For the population error, define
\[
\mathcal E(a):=
2^{-k}\sum_{s\in\{\pm1\}^k}\bigl(a^\top \Phi(s)-\tilde f(s)\bigr)^2.
\]
Since both $f$ and $\Phi$ depend only on the first $k$ coordinates, $\mathcal E(a)$ is exactly the uniform $\{\pm1\}^d$-risk of the network with frozen first layer.
Further, since $\nabla L_\lambda(a_\lambda^\star)=0$ and $L_\lambda$ is $M_\lambda$-smooth,
\[
L_\lambda(a)
\le
L_\lambda(a_\lambda^\star)+\frac{M_\lambda}{2}\|a-a_\lambda^\star\|_2^2.
\]
Also, by definition of $L_\lambda$,
\[
\frac12 \mathcal E(a)\le L_\lambda(a).
\]
Combining the last two displays with \eqref{eq:ridge-comparison}, we obtain
\[
\mathcal E(a)
\le
\lambda\|a^*\|_2^2 + M_\lambda\|a-a_\lambda^\star\|_2^2.
\]
Taking $a=a^{(T_2)}$, expectations, and using \eqref{eq:even-distance} together with \eqref{eq:sigma-lambda-star} gives \eqref{eq:phaseII-pop-risk}.
\end{proof}
We now choose $\lambda$ and $T_2$. Set
\begin{equation}\label{eq:lambda-choice}
\lambda := \frac{\delta}{3\|a^*\|_2^2}.
\end{equation}
Then the first term on the right-hand side of \eqref{eq:phaseII-pop-risk} is at most $\delta/3$. It is therefore enough to choose $T_2$ so that
\begin{equation}\label{eq:T2-choice}
2M_\lambda e^{-T_2/\kappa_\lambda}\|a^{(0)}-a_\lambda^\star\|_2^2 \le \frac{\delta}{3}
\qquad\text{and}\qquad
\widetilde{\mathcal O}\!\left(
\frac{M_\lambda^3\tau_{\rm mix}\|a^*\|_2^2}{\lambda^2 T_2}
\right)
\le \frac{\delta}{3}.
\end{equation}
Equivalently, one may take
\[
T_2
=
\widetilde{\mathcal O}\!\left(
\kappa_\lambda \log\!\Bigl(\frac{M_\lambda\|a^{(0)}-a_\lambda^\star\|_2^2}{\delta}\Bigr)
+
\frac{M_\lambda^3\tau_{\rm mix}\|a^*\|_2^6}{\delta^3}
\right).
\]
Since $\tau_{\rm mix}=O(dk/p)$ by \eqref{eq:phaseII-mixing}, and since by the certificate lemma
\[
\|a^*\|_2^2 \le 2^k \|a^*\|_\infty^2 = O_k(\varepsilon^{-2}),
\]
this yields a linear bound in $d$ and polynomial bound in $1/\delta$.

	\subsection{Finishing the Proof}
    \begin{proof}[Proof of Theorem \ref{thm:main_juntas_fd}]
        We work with $p=\frac{1}{2}$ in the lazy random walk. Thus, taking $B \geq C\frac{k^{15}d}{\varepsilon^2\tau^3}$ and $\gamma_1 = \frac{\sqrt{2Bd}}{\kappa\sqrt{k}}$, Proposition \ref{prop:non-degeneracy} applies and at the end of Phase I, the weights vector $w^{(1)}$ satisfies \eqref{eq:non-degeneracy} with the given probability. Moreover, the expression \ref{eq:GT-coordinate}, and the variance bound from Lemma \ref{lem:Yvar}, make it clear that with high probability $\|w^{(1)}\|_1 = O_k\left(\frac{\|f\|_{\infty}}{\varepsilon}\right)$, and so with our choice of $N$ and $A$, Lemma \ref{lem:certificate} applies as well. So, after sampling the biases $\hat{b}$, with probability $1-C\varepsilon$, there exists $a^*$, a weights vector for the second layer, satisfying \eqref{eq:astar}. 
        The proof now concludes by applying Proposition \ref{prop:phaseII-even} and using the choices for $T_2$, $\gamma_2$, and $\lambda$.
    \end{proof}
	
	\section{Anti-concentration for additive functionals of Markov chains (Proof of Lemma~\ref{lem:BEedges})} \label{sec:ACmarkov}
	We consider the \emph{$p$-lazy edge random walk} of the discrete cube, defined as follows:
	The state space is $\Omega_k = \{(u,v) \in \{-1,1\}^k\times\{-1,1\}^k :\ \|u-v\|_1 \leq 2 \}$, i.e. the set of all oriented edges, augmented by self-loops, of the discrete $k$-dimensional cube. For the evolution of the chain, assume that for some $t\geq 0$, $Y_t = (u_t,v_t)$, then $Y_{t+1}$ is obtained by sampling performing one step of a lazy random walk starting from $v_t$. Call the new vertex $v_{t+1}$ and set $Y_{t+1} = (v_t, v_{t+1})$. In other words $Y_t$ is a random walk on the (directed) edge graph of $\{-1,1\}^n$, where the laziness is represented by the probability to move to, or stay at, a self-loop. With that perspective if $(u,v) \sim\pi_E$, the unique invariant measure, then $u \sim \mathrm{uniform}(\{-1,1\}^n)$ and $v$ is obtained from $u$ via a one-step $p$-lazy random walk.

	We shall require the following definitions: Let $C(\Omega_k)$ be the space of real-valued functions over $\Omega_k$. For $t\geq 0$, define the Markov semi-group $P_t$ by
	$$P_tf(y) = \E\left[f(Y_t)| Y_0 = y\right]\qquad \forall f\in C(\Omega_k).$$
	Further, if $\|\cdot\|$ is any norm on $C(\Omega_k)$, we say that $P_t$ is $\rho$-contracting with respect to $\|\cdot\|$, if for every $f \in C(\Omega_k)$ such that $\int f \pi_E = 0$,
	$$\|P_1f\| \leq \rho\|f\|,$$ 
	which also implies for every $t\geq 0$,
	$$\|P_tf\| \leq \rho^t\|f\|.$$
	On $C(\Omega_k)$ we can define the following norm $\|f\|_{\infty,L} :=\|f\|_{\infty} + \mathrm{Lip}(f),$ where $\|f\|_{\infty} = \max\limits_{x \in \Omega_k} |f(x)|$ and $\mathrm{Lip}(f)$ is the Lipschitz constant of $f$ with respect to the Hamming distance.
	
	It is a fact that $P_t$ is contracting for $\|\cdot\|_{\infty,L}$, see \cite{kloeckner2021toy} for the statement on the vertex random walk, the extension to the edge random walk is straightforward.
	\begin{lemma}[\cite{kloeckner2021toy}] \label{lem:contraction}
		The semigroup $P_t$ is $\rho$-contracting with respect to $\|\cdot\|_{\infty,L}$, with $\rho = 1 - \frac{2p}{k^2}$.
	\end{lemma}
	
	Now, fix $f \in C(\Omega_k)$ and initialize $Y_1 \sim \pi_E$ according to the invariant measure. For $T \in\mathbb{N}$ we define the the additive functional 
	$$ S_T:= \frac{1}{\sqrt{T}\sigma_f}\sum\limits_{t=1}^T \left(f(Y_t) - \mathbb{E}\left[f(Y_1)\right]\right),$$
	where $\sigma_f^2$ is the \emph{dynamical variance} defined as
	$$\sigma_f^2 := \lim_{T\to \infty} \frac{1}{T}\mathrm{Var}\left(\sum\limits_{t=1}^T f(Y_t)\right) = \mathrm{Var}(f(Y_1)) + 2\sum\limits_{t=1}^\infty\mathrm{Cov}(Y_1,Y_t).$$
	Supposing that $0< \sigma_f <\infty$, since the process $Y_t$ is stationary we see that for every $T \in \mathbb{N}$, $\E[S_T] = 0$ and $\mathrm{Var}(S_T) \leq  \lim_{T \to \infty} \mathrm{Var}(S_T) = 1.$
	We should thus expect to have a central limit theorem for $S_T$. Indeed, this is the content of \cite[Theorem C]{kloeckner2019effective} with a quantitative error bound.
	\begin{proposition} \label{prop:markovCLT}
		Let $f \in C(\Omega_k)$ satisfy $\|f\|_{\infty} \leq M$, for some $M \geq 1$ and let $G$ be a standard Gaussian random variable. Then, 
		$$\sup\limits_{x\in \mathbb{R}}\left|\mathbb{P}\left(S_T \leq x\right) - \mathbb{P}\left(G \leq x\right)\right| \leq C\frac{k^4M^3}{p^2\min(\sigma_f^{3},\sigma_f)\sqrt{T}}.$$
		
	\end{proposition}
	\begin{proof}
		First, define the function $\tilde{f} = \frac{f - \mathbb{E}[f(Y_1)]}{\sigma_f}$. According to \cite[Theorem C]{kloeckner2019effective} there exists a universal constant $C> 0$, such that
		$$\sup\limits_{x\in \mathbb{R}}\left|\mathbb{P}\left(S_T \leq x\right) - \mathbb{P}\left(G \leq x\right)\right| \leq C\frac{1}{\sqrt{T}(1-\rho)^2}\max\left(\|\tilde{f}\|_{\infty,L},\|\tilde{f}\|_{\infty,L}^3\right),$$
		where $\rho$ is the contraction coefficient with respect to $\|\cdot\|_{\infty,L}$.
		By Lemma \ref{lem:contraction}, $1-\rho = \frac{2p}{k^2}$.
		Further since $\left|\E[f(Y_1)]\right|  \leq \|f\|_{\infty}$ and $\mathrm{Lip}(f)\leq 2\|f\|_{\infty}$, we have $\|\tilde{f}\|_{\infty,L} \leq \frac{4\|f\|_{\infty}}{\sigma_f} \leq \frac{4M}{\sigma_f}$. The proof is complete
	\end{proof}
	Since the standard Gaussian density is bounded by $\frac{1}{\sqrt{2\pi}}$, the following corollary is now immediate.
	\begin{corollary} \label{corr:markovAC}
		Under the same conditions of Proposition \ref{prop:markovCLT}, let $\varepsilon > 0$. If $$T \geq C\varepsilon^{-2}\frac{M^6k^8}{p^4\min(\sigma_f^6,\sigma_f^2)},$$ for some universal constant $C >0$, then for any $a \in \mathbb{R}$,
		$$\mathbb{P}\left(|S_T -a|\leq \varepsilon\right) \leq \varepsilon.$$
	\end{corollary}
	
	Let us also apply Proposition \ref{prop:markovCLT} to the function $\varphi_v$ from Lemma \ref{lem:BEedges}.
	\begin{proof}[Proof of Lemma \ref{lem:BEedges}]
	The lemma is an immediate consequence of Proposition \ref{prop:markovCLT} by noting that since $v_j \in \{-2,0,2\}$,
	\[
	\left|\varphi_{v}(x,y)\right|
	= \left|\bigl(f(x^{(u)})-f(x^{(u-1)})\bigr)\cdot \sum_{j=1}^k (x^{(u)}_j-x^{(u-1)}_j) v_j\right| \leq 8\|f\|_{\infty}k,
	\]
    and so we can substitute $M \leq 8\|f\|_{\infty}k$.
	\end{proof}

\section{Proof of Theorem~\ref{thm:lower-bound}}
For fully connected architectures with i.i.d.\ initialization, and for SGD updates that treat all input coordinates symmetrically, the training dynamics are equivariant under permutations of the input coordinates. Consequently, all targets in the orbit \(\orb(f)\) induce the same law for the trained predictor up to the corresponding input permutation, and are therefore equally hard to learn. It is thus enough to bound the average performance over \(F\sim P_{\orb(f)}\).

Therefore, fix $F \sim P_{\orb(f)}$.
As in the proof of Theorem~3 of \cite{AS20}, we compare the output law under true labels and under random labels. 
Given a target function $F:\{\pm1\}^d\to\{\pm1\}$, define the update of the noisy-SGD algorithm (see Def.~\ref{def:noisy-SGD}) 
\[
\theta_F^{(t)}
=
\theta_F^{(t-1)}
- 
\frac1B\sum_{s=1}^B G_{t-1}(\theta_F^{(t-1)},x_{(t-1)B+s},F(x_{(t-1)B+s}))
+
Z^{(t)} ,
\]
where $G_{t-1} = \gamma [\nabla L]_A$ is the clipped gradient vector associated with the loss, $Z^{(t)} \sim \cN(0,\gamma^2 \tau^2 I_M)$ is the gradient noise and $\{x_s\}_{s \geq 1}$ is generated by a $p$-lazy random walk.
Let $\theta_{\ast}^{(t)}$ denote the \emph{label-averaged} reference process,
defined by the same recursion as $\theta_F^{(t)}$ but with each mini-batch
gradient replaced by its conditional expectation over a uniform label
$y\in\{\pm1\}$. By construction $\theta_{\ast}^{(t)}$ is independent of $F$.

For each $t\in[T]$ and $H\in\{F,\ast\}$, let $Q_H^{(t)}$ denote the conditional law of $\theta_H^{(t)}$
given the target function $F$ and the batches $(S_B^{(1)},\dots,S_B^{(t)})$.
By the standard total variation coupling argument,
\[
\Pr\bigl(\sign(\NN(x;\theta_F^{(T)}))=F(x)\bigr)
\le
\Pr\bigl(\sign(\NN(x;\theta_{\ast}^{(T)}))=F(x)\bigr)
+
\E_{F,S_B^T}\TV\!\left(Q_F^{(T)},Q_{\ast}^{(T)}\mid F,S_B^T\right).
\]
Since under the junk process the final predictor depends only on junk labels and the input law is balanced,
\[
\Pr_{x,F}\bigl(\sign(\NN(x;\theta_{\ast}^{(T)}))=F(x)\bigr)\le \frac12+o_d(1).
\] 
It therefore suffices to bound the total variation term. Fix $t\in[T]$. 
Define the one-step mean gradients
\[
\Gamma_{F}^{(t)}
:=
\frac1B\sum_{i=1}^B G_{t-1}(\theta_{\ast}^{(t-1)},x_i^{(t)},F(x_i^{(t)})),
\qquad
\Gamma_{\ast}^{(t)}
:=
\frac1B\sum_{i=1}^B \E_{Y\sim\Unif(\{\pm1\})}
G_{t-1}(\theta_{\ast}^{(t-1)},x_i^{(t)},Y),
\]
where $x_i^{(t)} = x_{(t-1)B+i}$ and for $h \in \{F,\ast\}$ we denote by $Q_{\ast,h}^{(t)}$ the distribution, given $F$ and $S_B^{(t)}$, of $\theta_{\ast,h}^{(t)} $ defined as:
\begin{align*}
   \theta_{\ast,h}^{(t)} = \theta_{\ast}^{(t-1)} - \Gamma_{h}^{(t)} + Z^{(t)}.
\end{align*}

Using the triangle inequality and data-processing inequality, as in~\cite{AS20}[Proof of Theorem 3], we can write:
\begin{align*}
    \TV\!\left(Q_F^{(t)},Q_{\ast}^{(t)}\mid F,S_B^t\right) \leq \TV\!\left(Q_F^{(t-1)},Q_{\ast}^{(t-1)}\mid F,S_B^{t-1}\right) + \TV\!\left(Q_{\ast,F}^{(t)},Q_{\ast,\ast}^{(t)}\mid F,S_B^{(t)}\right).
\end{align*}

Conditioned on $F$ and $S_B^{(t)}$, the one-step updates differ only by the shift in the Gaussian mean, hence
\[
\TV\!\left(Q_{\ast,F}^{(t)},Q_{\ast,\ast}^{(t)}\mid F,S_B^{(t)}\right)
\le
\frac{1}{2\gamma \tau}\|\Gamma_F^{(t)}-\Gamma_{\ast}^{(t)}\|_2,
\]
exactly as in \cite[Section~5.1]{AS20}.

\noindent
Taking expectation over $F,S_B^{(t)}$ and using Cauchy--Schwarz,
\[
\E_{F,S_B^{(t)}}
\TV\!\left(Q_{\ast,F}^{(t)},Q_{\ast,\ast}^{(t)}\mid F,S_B^{(t)}\right)
\le
\frac{1}{2\gamma \tau}
\left(
\E_{F,S_B^{(t)}}\|\Gamma_F^{(t)}-\Gamma_{\ast}^{(t)}\|_2^2
\right)^{1/2}.
\]

\noindent
We now bound the squared norm on the right-hand side. Since the samples within each batch are correlated, the tensorization argument of~\cite{AS20} does not yield a useful bound in our setting. We therefore adopt a different approach, based on a direct kernel-norm estimate.
For each coordinate $e\in[M]$, write
\[
g_e(x,y):=\bigl(G_{t-1}(\theta_{\ast}^{(t-1)},x,y)\bigr)_e. \]
Since $y\in\{\pm1\}$ is uniform, set
\[
\tilde g_e(x):=\frac12\bigl(g_e(x,1)-g_e(x,-1)\bigr).
\]
Then
\[
g_e(x,F(x))-\E_{y\sim\Unif(\{\pm1\})}g_e(x,y)=F(x)\tilde g_e(x),
\]
so
\[
\Gamma_{F,e}^{(t)}-\Gamma_{\ast,e}^{(t)}
=
\frac1B\sum_{i=1}^B F(x_i^{(t)})\tilde g_e(x_i^{(t)}).
\]
Therefore
\[
\E_{F\sim P_{\orb(f)}}\bigl(\Gamma_{F,e}^{(t)}-\Gamma_{\ast,e}^{(t)}\bigr)^2
=
\E\Biggl[
\frac1{B^2}\sum_{i,j=1}^B
K(x_i^{(t)},x_j^{(t)})
\tilde g_e(x_i^{(t)})\tilde g_e(x_j^{(t)})
\Biggr],
\]
where we introduced the kernel
\[
K(x,x')
:=
\E_{F\sim P_{\orb(f)}}[F(x)F(x')].
\]

Fix the batch $(x_i^{(t)})_{i=1}^B$, and write
\[
a_i:=\tilde g_e(x_i^{(t)}),
\qquad
K_{ij}:=K(x_i^{(t)},x_j^{(t)}).
\]
Then
\[
\frac1{B^2}\sum_{i,j=1}^B K_{ij}a_i a_j
=
\frac1{B^2}\langle aa^\top,K\rangle_F,
\]
where $\langle\cdot,\cdot\rangle_F$ denotes the Frobenius inner product. By Frobenius Cauchy--Schwarz,
\[
|\langle aa^\top,K\rangle_F|
\le
\|aa^\top\|_F\,\|K\|_F.
\]
Since $\|aa^\top\|_F=\sum_{i=1}^B a_i^2$, we obtain
\[
\frac1{B^2}\sum_{i,j=1}^B K_{ij}a_i a_j
\le
\left(\frac1B\sum_{i=1}^B a_i^2\right)
\left(
\frac1{B^2}\sum_{i,j=1}^B K_{ij}^2
\right)^{1/2}.
\]
Taking expectation and applying Cauchy--Schwarz once more yields
\[
\E_F\bigl(\Gamma_{F,e}^{(t)}-\Gamma_{\ast,e}^{(t)}\bigr)^2
\le
\Biggl(
\E\Bigl[\Bigl(\frac1B\sum_{i=1}^B \tilde g_e(x_i^{(t)})^2\Bigr)^2\Bigr]
\Biggr)^{1/2}
\cdot
\bigl(\CP_B^{\rm rw}(\cU_d,P_{\orb(f)})\bigr)^{1/2},
\]
where we denoted by
\begin{align*}
    \CP_B^{\rm rw}(\cU_d,P_{\orb(f)})
    &:=
    \E_{(x_i^{(t)})_{i=1}^B}\Biggl[
    \frac{1}{B^2}\sum_{i,j=1}^B K(x_i^{(t)},x_j^{(t)})^2
    \Biggr] \\
    &=
    \E_{(x_i^{(t)})_{i=1}^B}\Biggl[
    \frac{1}{B^2}\sum_{i,j=1}^B
    \Bigl(\E_{F\sim P_{\orb(f)}} [F(x_i^{(t)})F(x_j^{(t)})]\Bigr)^2
    \Biggr],
\end{align*}
with $(x_i^{(t)})_{i=1}^B$ generated by the lazy random walk with stationary distribution $\cU_d$ (Def.~\ref{def:lazy-rw}).
Defining
\[
\widetilde G_{t-1}(x)
:=
\frac12\Bigl(
G_{t-1}(\theta_{\ast}^{(t-1)},x,1)-G_{t-1}(\theta_{\ast}^{(t-1)},x,-1)
\Bigr).
\]
we have
\[
\E_{F,S_B^{(t)}}\|\Gamma_F^{(t)}-\Gamma_{\ast}^{(t)}\|_2^2
\le
\Biggl(
\E\Bigl[
\Bigl(
\frac1B\sum_{i=1}^B \|\widetilde G_{t-1}(X_i^{(t)})\|_2^2
\Bigr)^2
\Bigr]
\Biggr)^{1/2}
\cdot
\bigl(\CP_B^{\rm rw}(\cU_d,P_{\orb(f)})\bigr)^{1/2}.
\]
Hence
\[
\E_{F,S_B^{(t)}}
\TV\!\left(Q_{\ast,F}^{(t)},Q_{\ast,\ast}^{(t)}\mid F,S_B^{(t)}\right)
\le
\frac{1}{2\gamma \tau}
\Biggl(
\E\Bigl[
\Bigl(
\frac1B\sum_{i=1}^B \|\widetilde G_{t-1}(X_i^{(t)})\|_2^2
\Bigr)^2
\Bigr]
\Biggr)^{1/4}
\cdot
\bigl(\CP_B^{\rm rw}(\cU_d,P_{\orb(f)})\bigr)^{1/4}.
\]

Now use the clipping bound. Since
\[
\|G_{t-1}(\theta,x,y)\|_2\le \gamma A\sqrt M
\qquad
\text{for all }(\theta,x,y),
\]
we also have
\[
\|\widetilde G_{t-1}(x)\|_2
\le
\frac12\Bigl(
\|G_{t-1}(\theta_{\ast}^{(t-1)},x,1)\|_2
+
\|G_{t-1}(\theta_{\ast}^{(t-1)},x,-1)\|_2
\Bigr)
\le \gamma A\sqrt M.
\]
Therefore
\[
\Biggl(
\E\Bigl[
\Bigl(
\frac1B\sum_{i=1}^B \|\widetilde G_{t-1}(X_i^{(t)})\|_2^2
\Bigr)^2
\Bigr]
\Biggr)^{1/4}
\le
\gamma A\sqrt M,
\]
and thus
\[
\E_{F,S_B^{(t)}}
\TV\!\left(Q_{\ast,F}^{(t)},Q_{\ast,\ast}^{(t)}\mid F,S_B^{(t)}\right)
\le
\frac{A\sqrt M}{2 \tau}
\cdot
\bigl(\CP_B^{\rm rw}(\cU_d,P_{\orb(f)})\bigr)^{1/4}.
\]

Summing over $t=1,\dots,T$ and iterating the total variation recursion as in \cite[Proof of Theorem~3]{AS20}, we obtain
\[
\E_{F,S_B^T}\TV\!\left(Q_F^{(T)},Q_{\ast}^{(T)}\mid F,S_B^T\right)
\le
\frac{T A\sqrt M}{ 2 \tau}\,\bigl(\CP_B^{\rm rw}(\cU_d,P_{\orb(f)})\bigr)^{1/4}.
\]

Thus, to prove our result it remains to show that for large enough batch size $B$, $\CP_B^{\rm rw}$ is close to the standard $\CP$ as defined in Def.~\ref{def:CP}.

\begin{lemma}[Comparison with standard cross-predictability]
\label{prop:rw-cp-gap} 

Let $\{ x_t\}_{t \geq 1}$ be the stationary $p$-lazy random walk on $\{\pm1\}^d$.
Then,
\[
\CP_B^{\rm rw}(\cU_d,P_{\orb(f)})
\le
\CP(\cU_d,P_{\orb(f)})+\frac{d}{pB}.
\]
\end{lemma}

\begin{proof} 
Fix $F,F'$ and define $h(x):=F(x)F'(x)\in\{\pm1\}$.
Then
\[
\Bigl(\E_{F\sim P_{\orb(f)}}[F(x)F(x')]\Bigr)^2
=
\E_{F,F'\sim P_{\orb(f)}}[h(x)h(x')].
\]
Therefore,
\[
\CP_B^{\rm rw}(\cU_d,P_{\orb(f)})
=
\E_{F,F'}\E\Biggl[
\frac1{B^2}\sum_{i,j=1}^B h(x_i)h(x_j)
\Biggr]
=
\E_{F,F'}\E\Biggl[
\biggl(
\frac1B\sum_{i=1}^B h(x_i)
\biggr)^2
\Biggr].
\]

Let
\[
\bar h_B:=\frac1B\sum_{i=1}^B h(x_i).
\]
Since the random walk is stationary,
\[
\E[\bar h_B^2]
= \bigl(\E[\bar h_B]\bigr)^2+\Var(\bar h_B)
=
\bigl(\E_{\cU_d}h\bigr)^2+\Var(\bar h_B).
\]
Averaging over $F,F'$ gives
\[
\CP_B^{\rm rw}(\cU_d,P_{\orb(f)})
=
\CP(\cU_d,P_{\orb(f)})+\E_{F,F'}\Var(\bar h_B).
\]
It remains to bound the variance term. The $p$-lazy random walk on $\{\pm1\}^d$ is reversible with spectral gap
\(
2p/d.
\)
Since the walk is reversible, for every bounded observable $h$,
decomposing $h = \E_{\cU_d}[h] + h_0$ with $\E_{\cU_d}[h_0]=0$ gives
\[
\Cov(h(x_0),h(x_k)) = \langle h_0, P^k h_0\rangle_{\cU_d}
\le (1-\mathrm{gap})^k\Var(h(x_0)),
\]
where $P$ denotes the transition kernel of the $p$-lazy random walk.
Therefore
\begin{align*}
\Var(\bar h_B)
&=
\frac1{B^2}\sum_{i,j=1}^B \Cov(h(x_i),h(x_j))\\
&\le
\frac{\Var(h(x_0))}{B^2}
\left(
B+2\sum_{k=1}^{B-1}(B-k)(1-{\rm gap})^k
\right)\\
&\le
\frac{\Var(h(x_0))}{B}
\left(
1+2\sum_{k=1}^\infty (1-{\rm gap})^k
\right)\\
&\le
\frac{2}{B\cdot {\rm gap}}\Var(h(x_0)).
\end{align*}
Since $|h|\le 1$, we have $\Var(h(x_0))\le 1$, and thus
\[
\Var(\bar h_B)\le \frac{2}{B\cdot {\rm gap}}=\frac{d}{pB}.
\]
This proves the result.
\end{proof}

\section{Experiment Details}
All experiments were performed using the PyTorch framework~(\cite{paszke2019pytorch}) and they were executed on NVIDIA Volta V100 GPUs. Each experiment is repeated $5$ times and we plot the mean; shaded regions denote $\pm 1 $ standard deviation.

\paragraph{Architecture.} For the results presented in the main, we used a 4-layer MLP trained by SGD with all layers trained jointly (no layerwise training). In particular, we choose a fully-connected architecture of $3$ hidden layers of neurons of size $512, 512, 64$, and $\ReLU$ activation. 
We use PyTorch's default initialization, which initializes weights of each layer with $\Unif[\frac{1}{\sqrt{\rm dim}_{\rm in}},-\frac{1}{\sqrt{\rm dim}_{\rm in}}] $, where ${\rm dim}_{\rm in}$ is the input dimension of the corresponding layer. 

\paragraph{Training procedure.} We consider the square loss: $L_{\ell_2}(\hat y, y) : = (\hat y-y)^2$ or the TD loss, defined in Def~\ref{def:TDloss}, with $\alpha=0.9$. In the iid data case, we sample fresh batches of samples at each iterations from the uniform distribution on the hypercube. In the RW data case, we take samples from the trajectory of a lazy random walk with lazyness parameter $p=0.9$. We stop training either when the training loss is less than $0.01$, or when $10^6$ iterations are performed. 
We select test samples from $\Unif\{ \pm 1 \}^d$, for both parities and juntas. In all experiments, the test-set is of size $8192$.

\paragraph{Hyperparameter tuning.} The primary goal of our experiments is to conduct a fair comparison between RW and iid data, with either the TD loss or the square loss. Thus, we did not engage in extensive hyperparameter tuning. We tried different learning rates, and we did not observe significant qualitative differences. We chose to report the experiments obtained for batch size of $1$ and a learning rate of $0.005$.


\end{document}